\documentclass{article} % For LaTeX2e
\usepackage{iclr2027_conference,times}
\renewcommand{\headrulewidth}{0pt} % to remove the line at the top of the page

\newcommand{\norm}[1]{\left\lVert#1\right\rVert_2}
\usepackage{amsmath,amsfonts,bm}

\def\eqref#1{equation~\ref{#1}}
\def\1{\bm{1}}

\def\mI{{\bm{I}}}

\DeclareMathAlphabet{\mathsfit}{\encodingdefault}{\sfdefault}{m}{sl}
\SetMathAlphabet{\mathsfit}{bold}{\encodingdefault}{\sfdefault}{bx}{n}

\usepackage{hyperref}
\usepackage{url}
\usepackage{graphicx}
\usepackage{wrapfig}
\usepackage{amssymb}
\usepackage{booktabs}
\usepackage{xcolor}
\usepackage{bm}
\usepackage{array}
\usepackage{longtable}
\usepackage{fvextra}
\usepackage{multirow}
\usepackage{tabularx}
\usepackage{amsthm}
\usepackage[svgnames]{xcolor}

\newtheorem{proposition}{Proposition}
\newtheorem{lemma}{Lemma}
\newtheorem{definition}{Definition}
\newtheorem{corollary}{Corollary}

\title{Causal and Interpretable Structures in LLM Compositional Tasks}

\author{%
  \textbf{Gurbir Arora}\textsuperscript{$1$},\,
  \textbf{Toni J.B. Liu}\textsuperscript{$1$},\,
  \textbf{Jiajun Bao}\textsuperscript{$1$},\,
  \textbf{Raphaël Sarfati}\textsuperscript{$1,2$},\,
  \textbf{Christopher J. Earls}\textsuperscript{$1,2$}
\\
\\
  \textsuperscript{$1$}Cornell University, USA\qquad
  \textsuperscript{$2$}Goodfire AI, USA
\\
  \small{
    \textbf{Correspondence:} \href{mailto:g.arora@cornell.edu}{g.arora@cornell.edu}
  }
}

\iclrfinalcopy % Uncomment for camera-ready version, but NOT for submission.
\begin{document}

\maketitle

\begin{abstract}
Large language models are able to solve tasks whose answers depend on not only individual input tokens, but also on relations among them. How is such relational information represented and processed across transformer layers? We study activations from ensembles of prompts that require inferring relationships between three tokens corresponding to a cyclic concept (months, hours, weekdays, and musical notes) to correctly predict the next token. Across model families (Llama, Qwen, Gemma, and Mistral) and cyclic concepts, we find a consistent layerwise progression in how the joint dependence among the tokens is geometrically organized and causally used: intermediate layers use a joint representation based on the inferred relationship between two tokens, while later layers use a joint representation associated with all three tokens to correctly complete the task. We also find other relationships between tokens that are geometrically structured but remain causally inert in the next-token prediction. Crucially, when taken together, these geometric and causal investigations reveal the representation-level mechanism that progressively organizes and composes the relational information to form the answer. More surprisingly, restricting the models to such causally relevant joint representations improves next-token prediction accuracy.
\end{abstract}

\section{Introduction}
Many tasks require a large language model (LLM) to use relationships between multiple input tokens. Consider the prompt, ``\emph{Hi John, the conference is scheduled between January 13 and March 13. This is the same time as between May 13 and}''. The correct next-token prediction depends crucially on three tokens: \emph{January}, \emph{March}, and \emph{May}. Let $A$, $B$, and $C$ denote these three tokens, respectively, and define $\gamma=B-A\;(\mathrm{mod\;}12)$. The correct answer is then $D=C+\gamma\;(\mathrm{mod\;}12)$. The individual representations of $A$, $B$, and $C$ lie on the low-dimensional manifold associated with calendar months~\citep{engels2025not}. We ask whether computationally relevant quantities arising jointly from these tokens, such as $\gamma$ or $D$, also have geometric structure, and whether that structure is used in the next-token prediction.

We vary $A,B,C$ independently over calendar months while keeping everything else fixed. Across this ensemble of prompts, the last-token activation varies with $A$, $B$, and $C$ individually and with the particular combinations in which they co-occur. These sources of variation are superposed in the full activation ensemble. We use the functional analysis of variance (ANOVA) decomposition to separate the resulting activation ensemble into terms associated with individual variables, pairwise interactions, and their three-way interaction~\citep{hoeffding}. We treat the interaction terms as geometric objects and study their organization and causal role across the layers of an LLM.

We study this question across several cyclic concepts: months, hours, weekdays, and musical notes. The cyclic structure gives us known relational quantities. In particular, $\gamma=B-A$ is a pairwise relation, while the answer $D=C+\gamma$ depends jointly on all three inputs. Importantly, $\gamma$ is never supplied explicitly and must instead be inferred from $A$ and $B$. This allows us to ask how the joint dependence associated with $\gamma$ and $D$ is represented and used across network depth.

Our contributions are: \textbf{(a) Geometric structure.} We find that the pairwise interaction between $A$ and $B$ is organized by $\gamma$ in the middle layers. Other pairwise interactions are also organized by their corresponding pairwise differences. The three-way interaction between $A$, $B$, and $C$ becomes organized by $D$ at a later depth. \textbf{(b) Causality.} We show that the representation associated with $A,B$ is causally relevant in the middle layers, while the representation associated with $A,B,C$ becomes causally relevant at later layers. \textbf{(c) Improved performance.} Retaining only the causally relevant three-way interaction improves the performance on the task. \textbf{(d) Relationship transfer.} We show that the relationship encoded by interaction terms transfers across cyclic domains, such as from hours to months.

\subsection{Related work}
\paragraph{Geometry of concepts.}
Prior work has identified low-dimensional geometric structure associated with individual concepts in latent activations~\citep{engels2025not, modell2025originsrepresentationmanifoldslarge, park2025the, karkada2026symmetries, hu2026language, prieto2026from}. Other work has studied how representational geometry evolves during computation and how interventions along that geometry affect model behavior~\citep{gurnee2025when, wurgaft2026manifold, sarfati2026shapebeliefsgeometrydynamics}. We instead study the geometry and causality of joint dependence between multiple tokens.

\begin{figure}[t]
    \centering
    \includegraphics[width=\linewidth]{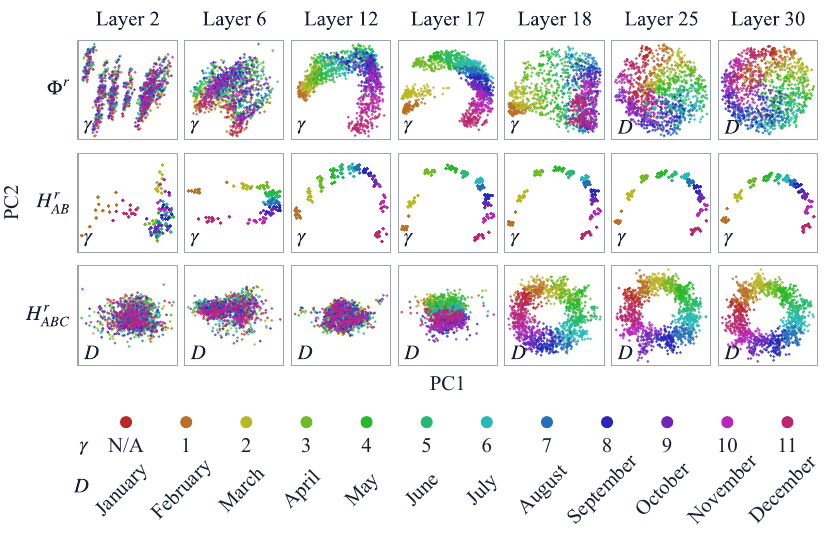}
    \caption{Two-dimensional visualization using the first two principal components for the activation ensemble $\Phi^r$ and its decomposed interaction ensembles $H_{AB}^r$ and $H_{ABC}^r$ for the prompt template in Section~\ref{sec:construction} using Llama-3.1-8B. Here, $\gamma=B-A\mod 12$ and the expected answer is $D=C+\gamma\mod 12$. Before fitting the principal components, we exclude prompts satisfying $A=B$, $B=C$, or $C=A$ (see Appendix~\ref{app:coincident-variables}). The bottom left of each panel shows the coloring according to either $\gamma$ or $D$. $H_{AB}^r$ gets organized by $\gamma$ in early layers. Until layer 17, $H_{ABC}^r$ is unstructured but becomes organized by $D$ at layer 18. See Section~\ref{sec:fourier} for a quantitative analysis.}
    \label{fig:decomposition}
\end{figure}

\paragraph{Relational representations and arithmetic computation.}
Prior work has shown that latent activations contain structured representations of relations between inputs, including approximately linear transformations, as well as distributed representations that bind entities or variables to relational roles~\citep{hernandez2024linearity, merullo-etal-2024-language, feng2024how, davies2023discoveringvariablebindingcircuitry, dai-etal-2026-cell, todd2025algebra, wang-etal-2024-locating}. Related work on arithmetic has identified structured numerical representations and computational mechanisms, including Fourier and periodic structure in transformers trained on modular arithmetic and in pretrained language models~\citep{nanda2023progress, furuta2024empiricalinterpretationinternalcircuits, zhou2024pretrained, kantamneni2025languagemodelsusetrigonometry, levy-geva-2025-language}. A closely related work is that of~\citet{feucht2026arithmeticwildllamauses}, who study tasks involving an explicit offset $\gamma$ and a concept $C$ from a cyclic domain. They identify a universal base-10 addition mechanism in Llama-3.1-8B by characterizing MLP neurons associated with the task's Fourier structure. In contrast, we work at the representation level. The relation \(\gamma\) in our setting is not supplied explicitly, but arises from the joint dependence of two independently varied inputs. We isolate the component of the residual-stream representation associated with this joint dependence and study its geometry and causal role across layers.

\paragraph{Causal interventions.}

The presence of structured or decodable information in an activation does not by itself establish that the model uses that information to perform the task. Causal interventions and activation patching have been used in several works to test the functional role of such representations~\citep{vig2020, geiger2021, geiger2024findingalignmentsinterpretablecausal, zhang2024towards, arora-etal-2024-causalgym}. Our experiments similarly intervene directly on joint representations of multiple input tokens.

\section{Setup}\label{sec:setup}
In the main text, all results are reported using base Llama-3.1-8B~\citep{llama_8b}, which contains 32 layers indexed by $\ell \in \{0,\ldots,31\}$. In the appendices, we reproduce all our results on base models: Llama-3.2-3B~\citep{meta2024llama32}, Qwen-2.5-7B~\citep{qwen2.5}, Qwen-3-8B~\citep{qwen3technicalreport}, Gemma-2-9B~\citep{gemma_2024}, Gemma-3-12B~\citep{gemma_2025}, Mistral-small-24B-Base-2501~\citep{mistral}. In each model, we analyze residual-stream activations after every transformer block.

\subsection{Construction of prompts}
\label{sec:construction}
Several other works have used controlled prompt variations to reveal manifolds associated with individual concepts~\citep{engels2025not, kantamneni2025languagemodelsusetrigonometry}. We propose a similar experimental design, but we aim to elicit a relationship between concepts. We consider several cyclic concepts---months, weekdays, hours, musical notes---and let $\mathcal{X}$ be the associated vocabulary. In the main text, we discuss the months domain, $\mathcal{X}:=\{\text{January},\ldots,\text{December}\}$. Consider the variables $(A, B, C) \in \mathcal{X}^3$ along with the prompt
% \Chris{the foregoing is a bit confusing to me...why introduce a dictionary, over the undefined vocabulary...are you defining it as a sub-set? not clear to me...I also find the overloading of ``domain'' to be confusing}:
\begin{verbatim}
Hi {name}, the {noun} is scheduled between {A} {dates} 
and {B} {dates}. This is the same time as between {C} {dates} and
\end{verbatim}
Here, \verb|name| is chosen from the most common names of the past century~\citep{ssa_babynames_century_2026}, \verb|noun| varies over synonyms of \verb|conference|, while \verb|dates| $\in \{1,\ldots,28\}$ (see Appendix~\ref{app:variables} for a list). We call each combination of $(\verb|name|,\verb|noun|,\verb|dates|)$ a \textit{replicate} and label it by $r$. These replicates define prompt variations that leave the underlying task unchanged. We further require that all replicate variables tokenize to the same length so that positional offsets do not introduce a confound.  We randomly subsample 10 such replicates. Unless stated otherwise, we average over the replicates when we report any quantity.

We design prompts associated with each cyclic concept such that the answer to the task is the next token. We measure the performance by computing the top-$1$ and top-$3$ accuracy, where top-$k$ accuracy is the fraction of prompts for which the correct answer features in the $k$ largest logits. The top-$1$ accuracy of Llama-3.1-8B is $62.2\% \pm 5.1\%$, and the top-$3$ accuracy is $85.4\% \pm 3.3\%$, where the variation is reported as the standard deviation across replicates $r$\footnote{The accuracies are computed over all prompts with distinct variables $(A,B,C)$. See Appendix~\ref{app:coincident-variables}.}. Because top-$3$ accuracy varies less with larger offsets of $\gamma=B-A\gtrsim6$ than top-$1$ accuracy (Appendix~\ref{app:capabilities}), we use top-$3$ accuracy as the primary performance metric in the main text. The analogous results using top-$1$ accuracy are reported in the appendices.

% \begin{table}[t]
% \centering
% \caption{
% Top-$1$ and top-$3$ accuracy of Llama-3.1-8B as a function of
% $\gamma=B-A \pmod{12}$. Values report mean accuracy across replicates $r$,
% with variation reported as $\pm$ standard deviation.
% }
% \label{tab:llama-accuracy}
% \resizebox{\linewidth}{!}{
% \begin{tabular}{c cccccccccccc c}
% \toprule
% $\gamma$
% & 0 & 1 & 2 & 3 & 4 & 5 & 6 & 7 & 8 & 9 & 10 & 11
% & Overall \\
% \midrule
% Top-$1$ (\%)
% & 100 $\pm$ 0 & 100 $\pm$ 0 & 100 $\pm$ 0 & 92.6 $\pm$ 0.9 & 87.0 $\pm$ 1.6 & 56.7 $\pm$ 2.5
% & 96.8 $\pm$ 0.6 & 20.0 $\pm$ 1.3 & 50.0 $\pm$ 2.9 & 23.5 $\pm$ 1.6 & 44.4 $\pm$ 11.9 & 55.3 $\pm$ 12.5
% & 68.85 $\pm$ 0.73 \\
% Top-$3$ (\%)
% & XX.X & XX.X & XX.X & XX.X & XX.X & XX.X
% & XX.X & XX.X & XX.X & XX.X & XX.X & XX.X
% & 85.61 $\pm$ 0.44 \\
% \bottomrule
% \end{tabular}
% }
% \end{table}
\subsection{Ensemble decomposition}\label{sec:decompose}
Let $\bm{\phi}^{r}_{\ell}(A,B,C)\in\mathbb{R}^{d_{\text{model}}}$ denote the last-token activation at layer $\ell$ for replicate $r$. We define the corresponding activation ensemble as
% \Chris{would we then want to call this a ``replicate ensemble''?}
% \Chris{before a replicate seemed to be an sampled triple A,B,C, but here you seem to be using it to denote the iterator, r...}
\begin{equation}\label{eq:ensemble}
    \Phi^{r}_{\ell}
    :=
    \left(
        \bm{\phi}^{r}_{\ell}(A,B,C)
    \right)_{(A,B,C)\in\mathcal{X}^3}.
\end{equation}
Since we report all quantities as a function of layer, we drop the label $\ell$ for brevity. We decompose each activation as
\begin{equation}\label{eq:decomp}
    \bm{\phi}^r(A,B,C) = \underbrace{\bm{\mu}^r}_{\text{mean}} + \underbrace{\bm{h}_{A}^r + \bm{h}_{B}^r + \bm{h}_{C}^r}_{\text{additive terms}} + \underbrace{\bm{h}_{AB}^r + \bm{h}_{BC}^r + \bm{h}_{CA}^r + \bm{h}_{ABC}^r}_{\text{interaction terms}},
\end{equation}

% \begin{equation}\label{eq:decomp}
% \bm{\phi}^r
% =
% \underbrace{
%   \textcolor{termMu}{\bm{\mu}^r}
% }_{\text{mean}}
% +
% \underbrace{
%   \textcolor{termA}{\bm{h}_{A}^r}
%   +
%   \textcolor{termB}{\bm{h}_{B}^r}
%   +
%   \textcolor{termC}{\bm{h}_{C}^r}
% }_{\text{additive terms}}
% +
% \underbrace{
%   \textcolor{termAB}{\bm{h}_{AB}^r}
%   +
%   \textcolor{termBC}{\bm{h}_{BC}^r}
%   +
%   \textcolor{termAC}{\bm{h}_{CA}^r}
%   +
%   \textcolor{termABC}{\bm{h}_{ABC}^r}
% }_{\text{interaction terms}},
% \end{equation}
where every term is defined as an average with respect to the ensemble $\Phi^{r}$:
\begin{equation}
    \begin{aligned}
    \bm{\mu}^r &:= \mathbb{E}_{ABC}(\bm{\phi}^r),\\
    \bm{h}_{A}^r &:= \mathbb{E}_{BC}(\bm{\phi}^r) - \bm{\mu}^r,\\
    \bm{h}_{AB}^r &:= \mathbb{E}_{C}(\bm{\phi}^r) - \bm{\mu}^r - \bm{h}_{A}^r - \bm{h}_{B}^r,\\
    \bm{h}_{ABC}^r &:= \bm{\phi}^r - \bm{\mu}^r - \bm{h}_{A}^r - \bm{h}_{B}^r - \bm{h}_{C}^r - \bm{h}_{AB}^r - \bm{h}_{BC}^r - \bm{h}_{CA}^r.
\end{aligned}
\end{equation}

% \begin{equation}
% \begin{aligned}
% \textcolor{termMu}{\bm{\mu}^r}
% &:= \mathbb{E}_{ABC}(\bm{\phi}^r),\\
% %
% \textcolor{termA}{\bm{h}_{A}^r}
% &:= \mathbb{E}_{BC}(\bm{\phi}^r)
% - \textcolor{termMu}{\bm{\mu}^r},\\
% %
% \textcolor{termAB}{\bm{h}_{AB}^r}
% &:= \mathbb{E}_{C}(\bm{\phi}^r)
% - \textcolor{termMu}{\bm{\mu}^r}
% - \textcolor{termA}{\bm{h}_{A}^r}
% - \textcolor{termB}{\bm{h}_{B}^r},\\
% %
% \textcolor{termABC}{\bm{h}_{ABC}^r}
% &:= \bm{\phi}^r
% - \textcolor{termMu}{\bm{\mu}^r}
% - \textcolor{termA}{\bm{h}_{A}^r}
% - \textcolor{termB}{\bm{h}_{B}^r}
% - \textcolor{termC}{\bm{h}_{C}^r}
% - \textcolor{termAB}{\bm{h}_{AB}^r}
% - \textcolor{termBC}{\bm{h}_{BC}^r}
% - \textcolor{termAC}{\bm{h}_{CA}^r}.
% \end{aligned}
% \end{equation}

Analogously to Equation~\ref{eq:ensemble}, we define the ensemble associated with decomposed terms as
\[
H_A^r
=
\left(\bm{h}_A^r\right)_{A\in\mathcal X},
\qquad
H_{AB}^r
=
\left(\bm{h}_{AB}^r\right)_{(A,B)\in\mathcal X^2},
\quad
\text{and}
\quad
H_{ABC}^r
=
\left(\bm{h}_{ABC}^r\right)_{(A,B,C)\in\mathcal X^3}.
\]
% \[
% \textcolor{termA}{H_A^r}
% =
% \left(\textcolor{termA}{\bm{h}_A^r}\right)_{A\in\mathcal X},
% \qquad
% \textcolor{termAB}{H_{AB}^r}
% =
% \left(\textcolor{termAB}{\bm{h}_{AB}^r}\right)_{(A,B)\in\mathcal X^2},
% \quad
% \text{and}
% \quad
% \textcolor{termABC}{H_{ABC}^r}
% =
% \left(\textcolor{termABC}{\bm{h}_{ABC}^r}\right)_{(A,B,C)\in\mathcal X^3}.
% \]
The remaining terms are defined by cyclic permutation of the expressions above. Intuitively, $\bm{\mu}^r$ is the center of the activation ensemble, $H_A^r$ captures the dependence on $A$ alone, $H_{AB}^r$ captures the joint dependence on $A$ and $B$ that remains after averaging over $C$ and removing their individual effects, and $H_{ABC}^r$ contains the residual three-way dependence after accounting for all lower-order terms. We refer to $H_{AB}^r$, $H_{BC}^r$, $H_{CA}^r$, and $H_{ABC}^r$ as the interaction ensembles.

We emphasize that Equation~\ref{eq:decomp} is an exact identity rather than an approximation. It is known as the functional ANOVA decomposition~\citep{hoeffding, mcbook}. The decomposed components are pairwise orthogonal under the ensemble-averaged inner product (Appendix~\ref{app:anova}). We do not apply the inferential machinery of ANOVA, since the associated questions of statistical significance are orthogonal to the goal of this paper. Figure~\ref{fig:decomposition} shows the first two principal components of $\Phi^r$, $H_{AB}^r$, and $H_{ABC}^r$. The remainder of the paper analyzes the interaction ensembles as geometric objects to study their structure and causal role across network depth.

\section{Geometric structure of interaction terms}\label{sec:fourier}
The ANOVA decomposition described in Equation~\ref{eq:decomp} breaks the ensemble into its constituent parts. Figure~\ref{fig:decomposition} visually suggests that $H_{AB}^r$ can be parametrized by $\gamma = B-A$ and similarly $H_{ABC}^r$ can be parametrized by $D=C+\gamma$. In this section, we first rigorously quantify the dependence of interaction terms on $\gamma$ or $D$ by exploiting the cyclic nature of $A,B,C$ using the discrete Fourier transform in Section~\ref{sec:dft}. Second, we analyze the norm of vectors in each interaction ensemble in Section~\ref{sec:energy} to check for ensembles that are highly organized but have negligible norm overall.

\subsection{Fourier Transform}\label{sec:dft}
We perform a discrete Fourier transform (DFT) of every interaction term for every replicate $r$. Consider $H_{AB}^r$ as an example. Its DFT is given by
\begin{equation}
    \widehat{\bm h}_{k_Ak_B}^r = \sum_{A,B}\bm h^r_{AB}e^{-2\pi i(k_A A + k_BB)/12},
\end{equation}
where we identify the months with $A,B \in \{0,\ldots,11\}$ and $k_A,k_B$ are the associated Fourier modes. Let $\gamma = B-A$. If $\bm h_{AB}^r$ depends only on $\gamma$, that is, $\bm h_{AB}^r\equiv \bm f(\gamma)$, then its DFT
\[
\widehat{\bm h}_{k_Ak_B}^r = \sum_{\gamma}\bm f(\gamma)e^{-2\pi ik_B\gamma/12}\sum_{A}e^{-2\pi i(k_A +k_B)A/12}
\]
vanishes unless $k_A +k_B =0\mod{12}$. Thus concentration on modes satisfying $k_A = -k_B$\footnote{Strictly, the condition is $k_A+k_B=0 \;\mathrm{mod}\;{12}$; all frequency equalities are understood modulo 12.} quantifies the extent to which $H_{AB}^r$ is organized according to $\gamma$. Therefore we define \textit{mode fraction} as
\begin{equation}\label{eq:mode}
m(H_{AB}^r) := \frac{\sum_{k_A=-k_B}\norm{\widehat{\bm h}^r_{k_Ak_B}}^2}{\sum_{k_A,k_B}\norm{\widehat{\bm h}^r_{k_Ak_B}}^2}.
\end{equation}
The mode fraction $m(H^r_{AB})\in [0,1]$ measures the degree of organization of $H_{AB}^r$ by $\gamma$, with $m=1$ when the ensemble depends on $\gamma$ alone. We define an analogous mode fraction for $H_{BC}^r$, $H_{CA}^r$, and $H_{ABC}^r$ to probe for their dependence on $\gamma' = C-B$, $\gamma''=A-C$, and $D = C + B-A$ respectively. For the latter, the modes under investigation are $(k_A,k_B,k_C) = (-k,k,k)$.

Dependence on $\gamma$ cannot occur in terms lower-order than $H_{AB}^r$, and dependence on $D$ cannot occur in terms lower-order than $H_{ABC}^r$. Thus, if the residual stream contains structure organized by $\gamma$ or $D$, that structure must reside in $H_{AB}^r$ or $H_{ABC}^r$, respectively. Consistently, we find large mode fractions for both interaction ensembles but at different depths. Importantly, we also find that all second-order interaction terms are strongly organized by their corresponding pairwise differences, and that $H_{ABC}^r$ undergoes a sharp increase in organization by $D$ around layers 17--18 (Figure~\ref{fig:fourier}(a)). At the same depth, the mode fractions of $H_{BC}^r$ and $H_{CA}^r$ begin to decrease, while $m(H_{AB}^r)$ remains large throughout the depth of the transformer. Results across models and domains appear in Appendix~\ref{app:mode-fraction}.

Control experiments detailed in Appendix~\ref{app:controls} suggest that the pairwise-difference organization of the second-order interaction ensembles can arise from the co-occurrence of the corresponding tokens. This is in agreement with~\citet{karkada2026symmetries}. In contrast, $H_{ABC}^r$ is not organized by $D$ when the task structure is removed from the prompt, even when all three tokens $A,B,C$ are present.

\begin{figure}[t]
    \centering
    \includegraphics[width=\linewidth]{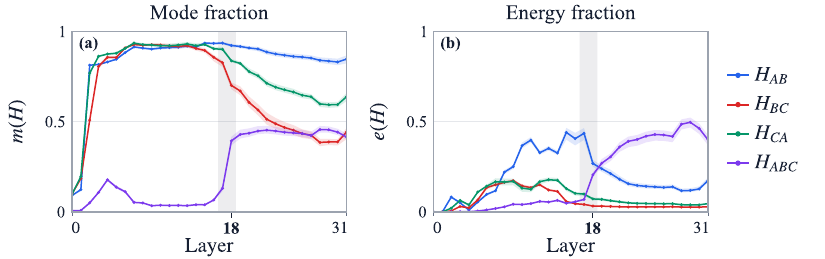}
    \caption{Mode fraction and energy fraction of the interaction ensembles across layers of Llama-3.1-8B. We zero out vectors for which $A$, $B$, and $C$ are not pairwise distinct before computing both fractions (Appendix~\ref{app:coincident-variables}). \textbf{(a)} Mode fraction measures the organization of $H_{AB}^r$, $H_{BC}^r$, and $H_{CA}^r$ by their corresponding pairwise differences, and of $H_{ABC}^r$ by $D=C+B-A$. Under a random null model, the expected mode fraction for second-order interactions is $0.09$ and that for third-order interaction is $0.008$. \textbf{(b)} Energy fraction measures the fraction of mean-subtracted activation energy attributable to each interaction ensemble. Throughout this work, colored bands show the central 80\% interval across 10 replicates, and vertical gray bands highlight the layers discussed in the text.
}
    \label{fig:fourier}
\end{figure}

\subsection{Energy fraction}\label{sec:energy}
An ensemble can have a high mode fraction even if all of its vectors are close to zero. Directly comparing the norms of interaction terms across layers is misleading because $\bm \phi^r$ grows substantially with depth, with most of this growth coming from the mean $\bm \mu^r$. Therefore we need a measure of relative magnitude for each ensemble across layers. For $S \in \{AB,BC,CA,ABC\}$, we define energy fraction $e(H^r_S)$ as\footnote{See Appendix~\ref{app:energy-fraction} for a relation between ANOVA decomposition and our definition of energy fraction.}
\begin{equation}
    e(H_{S}^r) = \frac{\mathbb{E}_{S}\Big(\norm{\bm{h}^r_{S}}^2\Big)}{\mathbb{E}_{ABC}\Big(\norm{\bm{\phi}^r(A,B,C)-\bm{\mu}^r}^2\Big)}
\end{equation}
Figure~\ref{fig:fourier}(b) shows the energy fraction of each interaction ensemble. The energy fraction of $H_{AB}^r$ peaks at layer 15 and starts decreasing around layers 17--18. Around the same depth, the energy fraction of $H_{ABC}^r$ increases dramatically. This coincides with the sharp increase in organization of $H_{ABC}^r$ by $D$, as found in the previous Section~\ref{sec:dft}. Together, these observations hint at a transition from a pairwise representation of the input relation \(B-A\) to a three-way representation aligned with the answer \(C+B-A\), which we test causally in the following section.

\section{Causal use of interaction terms}\label{sec:causal}
The previous section establishes that the interaction ensembles are highly organized by pairwise differences or by $D$, depending on the interaction ensemble and the layer. To test whether these ensembles are used in downstream computation, we intervene directly on the activation of the last token~\citep{zhang2024towards} using its ANOVA decomposition. In the first set of experiments in Section~\ref{sec:ablation}, we intervene on the last token at a particular layer to remove an interaction term. In Section~\ref{sec:replacement}, we instead remove all terms except for the mean and an interaction term. After each intervention, we let the model run without any further modification and score the performance.

\begin{figure}[t]
    \centering
    \includegraphics[width=\linewidth]{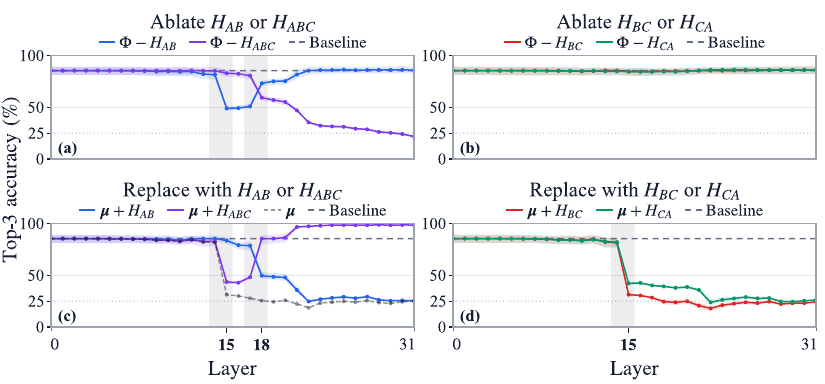}
    \caption{Top row shows the result of ensemble ablation experiments, while the bottom row shows corresponding ensemble replacement experiments. The horizontal dashed line shows the baseline performance. Dotted line at 25\% is the chance of getting the correct answer under a uniform prior. \textbf{(a)} Ablating $H_{AB}^r$ produces the largest decrease in top-3 accuracy around layers 15--17, whereas ablating $H_{ABC}^r$ sharply decreases accuracy after layer 17. \textbf{(b)} Ablating $H_{BC}^r$ or $H_{CA}^r$ has little effect on performance across model depth. \textbf{(c)} Replacement experiments show a complementary transition where retaining $H_{AB}^r$ preserves performance in the middle layers, while retaining $H_{ABC}^r$ preserves performance in later layers. \textbf{(d)} Retaining $H_{BC}^r$ or $H_{CA}^r$ sharply decreases performance from layer 15 onward.}
    \label{fig:ablation}
\end{figure}

\subsection{Ensemble Ablations}\label{sec:ablation}
We define \textit{ensemble ablation} of an interaction ensemble $H_S^r$ by subtracting the corresponding interaction vector $\bm h_S^r$ from the activation of the last token:
\begin{equation}
    \widetilde{\bm{\phi}}^r(A,B,C) = \bm \phi^r(A,B,C) - \bm h^r_S,
\end{equation}
for all $ A,B,C \in \mathcal{X}$ and $S \in \{AB,BC,CA, ABC\}$.

Figure~\ref{fig:ablation} shows the results of ensemble ablation experiments for every interaction term. We see that ablating $H_{BC}^r$ or $H_{CA}^r$ barely changes the top-$3$ accuracy relative to the baseline (Figure~\ref{fig:ablation}(b)). In contrast, ablating $H_{AB}^r$ or $H_{ABC}^r$ causes a drastic change in the model's performance, but at layers 15 and 18 respectively (Figure~\ref{fig:ablation}(a)).

\subsection{Ensemble Replacement}\label{sec:replacement}
In a complementary set of experiments, we perform \textit{ensemble replacement} by retaining a single interaction ensemble together with the mean. Concretely, we replace the activation of the last token with
\begin{equation}
    \widetilde{\bm \phi}^r(A,B,C) = \bm \mu^r+\bm h^r_S,
\end{equation}
for all $ A,B,C \in \mathcal{X}$ and $S \in \{AB,BC,CA,ABC\}$.

Figure~\ref{fig:ablation}(d) shows that retaining $H_{BC}^r$ or $H_{CA}^r$ alone induces a sharp decrease in top-3 accuracy around layers 14--15. However, Figure~\ref{fig:ablation}(c) shows that retaining $H_{AB}^r$ preserves performance relative to the unmodified baseline until around layer 17, after which performance decreases sharply. Because we intervene only at the last-token position and at a particular layer, the model may recompute task-relevant information downstream of the intervention. We therefore also include a $\bm\mu^r$-only baseline to measure this recovery. Before layer 17, retaining $H_{ABC}^r$ gives performance similar to the $\bm\mu^r$-only baseline. After layer 17, however, retaining $H_{ABC}^r$ substantially improves performance over the $\bm\mu^r$-only baseline and even exceeds the unmodified baseline.

Together, ensemble ablation and replacement experiments show a transition in causal relevance from $H_{AB}^r$ to $H_{ABC}^r$ around layers 17--18. Between layers 15--17, ablating $H_{AB}^r$ causes a sharp decrease in performance while retaining $H_{AB}^r$ preserves it. The same aforementioned properties hold for $H_{ABC}^r$ starting around layers 17--18. This transition occurs at the same depth where Section~\ref{sec:fourier} finds that $H_{ABC}^r$ becomes organized according to the answer $D$, the energy fraction of $H_{AB}^r$ starts decreasing, and the energy fraction of $H_{ABC}^r$ rises sharply. Although $H_{BC}^r$ and $H_{CA}^r$ are also strongly organized by their corresponding pairwise differences (Section~\ref{sec:fourier}), ablating them has little effect on the performance, and retaining them does not reproduce the layerwise replacement behavior of $H_{AB}^r$. Additional ablation and replacement experiments for different models and domains are in Appendices~\ref{app:ablation} and~\ref{app:replacement}, respectively. Appendix~\ref{app:causal-handoff} further shows that ablating $H_{AB}^r$ at layer 15 prevents the later emergence of $H_{ABC}^r$.

\section{Steering}\label{sec:steering}

So far, we have shown that the ensembles $H_{AB}^r$ and $H_{ABC}^r$ are causally relevant, but this does not tell us if the model uses their organization according to $\gamma$ or $D$. If the model indeed uses this organization, then changing $\gamma \to \gamma + \delta$ should shift the predicted answer by $\delta$, that is, $D \to D+\delta$. In this section, we test this by steering along $\gamma$ and $D$. There are multiple ways to steer using $H_{AB}^r$: shifting $B\to B+\delta$ such that $\bm h_{AB}^r \to \bm h_{A,B+\delta}^r$, shifting $A \to A-\delta$ such that $\bm h_{AB}^r \to \bm h_{A-\delta,B}^r$, or using an averaged vector that depends only on $\gamma$. Results using a combination of the first two strategies are reported in Appendix~\ref{app:steering}. Here, we describe the last strategy. We define an averaged vector $\bm{\bar h}^r_\gamma$ by\footnote{The chosen steering vector $\bm{\bar h}^r_\gamma$ is closely related to the definition of the mode fraction in Equation~\ref{eq:mode}. We show this correspondence in Appendix~\ref{app:mode-cluster}.}
\begin{equation}
\label{eq:steeringVec}
    \bm{\bar h}^r_\gamma:= \mathop{\mathbb{E}_{AB}}_{B-A=\gamma}\bm h_{AB}^r.
\end{equation}

To steer, we replace the activation $\bm\phi^r$ with $\bm{\widetilde\phi}^r$ according to
\begin{equation}\label{eq:steer}
    % \begin{split}
    \bm{\widetilde \phi}^{r}(A,B,C) = \bm\phi^r(A,B,C) + \alpha \left(\bm{\bar h}^r_{\gamma + \delta} - \bm{\bar h}^r_\gamma\right).
% \end{split}
\end{equation}

\begin{figure}
    \centering
    \includegraphics[width=\linewidth]{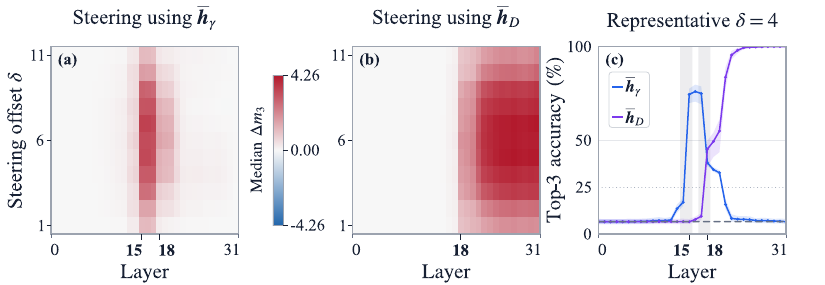}
    \caption{Steering using different vectors as a function of layer at $\alpha=1$. The median is evaluated among all the prompts for each layer and $\delta$.
    \textbf{(a)} We see that steering using  $\bm{\bar h}_\gamma$ is most effective between layers 15--17.
    \textbf{(b)} In contrast, steering using  $\bm{\bar h}_D$ becomes effective from layer 18 onward.
    \textbf{(c)} Top-3 accuracy for $\delta=4$ after steering. The dashed line shows the unmodified model's top-3 accuracy with respect to the target answer $D+\delta$.
    }
    \label{fig:steering}
\end{figure}

Similarly, we define the steering vector $\bm{\bar h}_D$ for the ensemble $H_{ABC}^r$ as
\begin{equation}
    \bm{\bar h}^r_D := \mathop{\mathbb{E}_{ABC}}_{C + B-A = D}\bm h_{ABC}^r
\end{equation}
and steer using Equation~\ref{eq:steer} after replacing $\bm{\bar h}^r_\gamma$ with $\bm{\bar h}^r_D$. The steering method is equivalent to the difference-of-means steering, analogous to that of~\citet{rimsky-etal-2024-steering} and~\citet{subramani-etal-2022-extracting} (see Appendix~\ref{app:steering-vector}).

Let $D^* = D+\delta$ be the target answer. To measure the effect of steering, we compute the top-3 logit margin $m_3$, defined as the difference between the target logit $z_{D^*}$ and the third-largest non-target logit $z_3$:
\begin{equation}
    m_3 = z_{D^*} - z_3.
\end{equation}
Thus $m_3 \geq 0$ iff the target logit is in top-3. Finally, we take the difference between steered and baseline $m_3$ to get
\begin{equation}
    \Delta m_3 = m_3^{\text{steered}} - m_3^{\text{baseline}}.
\end{equation}

In Figures~\ref{fig:steering}(a--b), we report $\mathrm{median\;}\Delta m_3$ for steering using $\bm{\bar h}_\gamma^r$ and $\bm{\bar h}_D^r$ at $\alpha=1$, and Figure~\ref{fig:steering}(c) shows the top-3 accuracy after steering for $\delta=4$. Steering using $\bm{\bar h}^r_{\gamma}$ becomes effective around layers 14--15, but around layers 17--18, its effect decreases and steering using $\bm{\bar h}^r_D$ becomes effective.

\section{Cross-domain transfer}\label{sec:transplant}

Our preceding experiments show that $H_{AB}^r$ is causally relevant in the middle layers while $H_{ABC}^r$ is relevant in the later layers. We now want to understand whether one interaction vector extracted from one context can substitute for the corresponding vector in another context performing the same underlying task. Prior work has shown that relational and task-level representations can be reused across contexts~\citep{wang-etal-2024-locating, todd2024function}. We focus on the interaction ensemble $H_{AB}^r$ here; the performance of transplanting $H_{ABC}^r$ varies across model and domain combinations. Appendix~\ref{app:cross-domain} shows that $H_{ABC}^r$ can be transplanted across domains when their subspaces are aligned. We will use \textit{domain-1} to refer to the original set of prompts that we have discussed so far and we use \textit{domain-2} for the following set of prompts:
\begin{verbatim}
Hi {name}, the {noun} is scheduled between {day} {A} o'clock and
{day} {B} o'clock. This is the same time as between {day}
{C} o'clock and {day} 
\end{verbatim}
where $A,B,C \in \mathcal{X}^{(2)} = \{1,\ldots,12\}$ and (\verb|name|, \verb|noun|, \verb|day|) are replicate variables (Appendix~\ref{app:variables}). We then decompose the domain-2 ensemble according to Section~\ref{sec:decompose}. We transplant an interaction ensemble by replacing each domain-1 interaction vector with the corresponding domain-2 interaction vector indexed by the same values of the underlying cyclic variables. We consider
\begin{equation}\label{eq:transAblate}
    \widetilde{\bm{\phi}}^{r_1,(1)}(A,B,C) = \bm{\phi}^{r_1,(1)}(A,B,C) -\bm h^{r_1,(1)}_{AB} + \bm h^{r_2,(2)}_{AB},
\end{equation}
where superscripts $(1), (2)$ refer to the domain and $r_1,r_2$ refer to the corresponding replicate. Analogous to the ensemble replacement experiments in Section~\ref{sec:replacement}, we also consider
\begin{equation}\label{eq:transReplace}
    \widetilde{\bm{\phi}}^{r_1,(1)}(A,B,C) = \bm{\mu}^{r_1,(1)} + \bm h_{AB}^{r_2,(2)}.
\end{equation}
In both cases, the model runs unmodified after the intervention and we score relative to domain-1. Figure~\ref{fig:transplant}(a) shows that replacing an ablated domain-1 interaction ensemble with the corresponding domain-2 ensemble restores the performance lost under ablation. Figure~\ref{fig:transplant}(b) shows that the domain-2 interaction ensemble also reproduces the layerwise trend observed in the replacement experiment in Section~\ref{sec:replacement}. 

\begin{figure}[t]
    \centering
    \includegraphics[width=\linewidth]{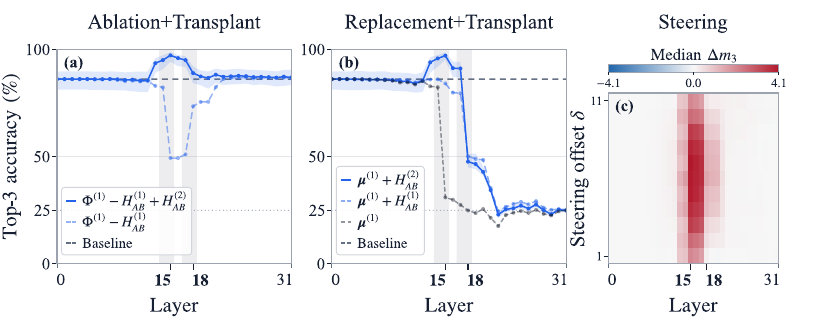}
    \caption{Shaded bands show the central 80\% interval across 10 randomly sampled replicate pairs $(r_1,r_2)$. \textbf{(a)} Transplanting the interaction ensemble $H_{AB}^{(2)}$ from domain-2 to domain-1 preserves performance. \textbf{(b)} Transplanting and retaining only the domain-1 mean $\bm{\mu}^{(1)}$ and domain-2 interaction ensemble $H^{(2)}_{AB}$ also preserves performance compared to the dashed blue within-domain replacement experiments. \textbf{(c)} Cross-domain steering using $\bm{\bar h}_{\gamma+\delta}^{r_2,(2)}-\bm{\bar h}_{\gamma}^{r_1,(1)}$ at $\alpha=1$.}
    \label{fig:transplant}
\end{figure}

We next ask whether the organization according to $\gamma$ also transfers across domains. Using the steering vector defined in Equation~\ref{eq:steeringVec}, we steer using the difference between the domain-1 vector for $\gamma$ and the domain-2 vector for $\gamma+\delta$:
\begin{equation}
\widetilde{\bm\phi}^{r_1,(1)}(A,B,C) =\bm\phi^{r_1,(1)}(A,B,C)+\alpha\left(\bm{\bar h}_{\gamma+\delta}^{r_2,(2)}-\bm{\bar h}_{\gamma}^{r_1,(1)}\right).
\end{equation}
If the organization by $\gamma$ transfers between two domains, then the answer should shift accordingly. Figure~\ref{fig:transplant}(c) shows that steering vectors constructed across the two domains produce a steering pattern similar to that in Section~\ref{sec:steering}.

\section{Discussion}\label{sec:discussion}
We started by asking whether it is possible to separate variation due to individual concepts from variation arising through their joint dependence in an activation ensemble. We showed that using the functional ANOVA decomposition is a simple yet powerful technique to accomplish this task. The resulting interaction ensembles exhibit clear geometric structure, but not all geometrically organized ensembles are causally necessary for the task. For the ensembles that are causally relevant, their organization can be exploited to predictably steer the model.
% \Chris{some people will only read the discussion and so it might be nice to recapitulate a few concrete results here.}
% \Chris{I might delay this to limitations?}

We also found two results that we did not anticipate. First, retaining only the mean and the third-order interaction can outperform the unmodified model. Our interventions show that $H_{ABC}^r$ is used by the LLM to output the correct answer, while other interaction ensembles either are unused or perform a different function. Therefore, it may happen that those other interaction ensembles interfere destructively in the computation in the last few layers. Second, transplanting $H_{AB}$ between domains that perform the same underlying task recovers much of the performance lost under ablation. This suggests that the intermediate representation of the inferred relation is sufficiently compatible across domains to support downstream computation. Neither observation would have been accessible in the full activation ensemble without first isolating the joint dependence from the remaining components.
% \Chris{this seems to me to be be a strong point to lead with in the introduction...the reason for the work, and an important part of the contributions}.

Our results also suggest a possible connection to the neuron-level analysis of \citet{feucht2026arithmeticwildllamauses}. They identify a small set of MLP neurons at layer 18 that are associated with Fourier components of the arithmetic representation in Llama-3.1-8B. In our analysis, the same depth is where the three-way interaction $H_{ABC}^r$ sharply increases in energy fraction, becomes organized by $D$, and becomes causally relevant. While we do not analyze individual neurons here, the interaction decomposition may provide a way to identify the layers in which such structure emerges and may also provide a model-agnostic starting point for neuron-level analyses.

\subsection*{Limitations}
\textbf{Scope}. We have studied a controlled setting with three variables $A,B,C$ that are drawn from the \emph{same cyclic} concept whose composition has a \emph{definitive} arithmetic answer. A natural extension is to allow $A,B,C$ to vary over different, potentially non-cyclic concepts whose composition may not have a well-defined answer.

\textbf{Exponential scaling}. As we increase the number of variables $p$, the number of prompts required to furnish the ANOVA decomposition scales exponentially~$\sim |\mathcal{X}|^p$.

\subsection*{Acknowledgments}
This research is funded in part by the Gordon and Betty Moore Foundation through Grant GBMF13901 to Cornell to support the work of G.A.
\subsection*{AI use statement}
% We have not used generative AI for developing theoretical models or conceptual frameworks, formulating mathematical claims, providing critical ingredients for proving mathematical claims, assisting in the writing of proofs, proposing or refining hypotheses, designing or providing feedback on research methodology or experiments, implementing methods, supporting qualitative and thematic data analysis, interpreting results. The following does not apply to this work: generating synthetic data sets, assisting with translation, cleaning and reformatting datasets.

We have used generative AI (GPT-5.6 and GPT-6) to code, proofread the paper to improve readability, correct typos and grammatical errors.
We take responsibility for the final content of this work, including text, claims or artifacts produced with the aid of generative AI.

\bibliography{iclr2027_conference}
\bibliographystyle{iclr2027_conference}

\appendix
\section{Models, Prompts, and Capabilities}
\label{app:capabilities}
In this work, we evaluate the robustness of our results across seven base models ranging from 3B to 24B parameters: Llama-3.2-3B, Llama-3.1-8B, Qwen-2.5-7B, Qwen-3-8B, Gemma-2-9B, Gemma-3-12B, Mistral-small-24B-Base-2501. Table~\ref{tab:frozen-prompts} lists all the prompts for each model, while Table~\ref{tab:model-domain-accuracy} shows their corresponding performances using top-1 and top-3 accuracy as metrics. Here, we report performance using all prompts, including cases with coincident variables, that is, when $A=B$, $B=C$, or $C=A$.

We use two conditions to fix the prompt template: \textbf{(a)} The model should be capable of answering the problem via the next token. \textbf{(b)} All variations over $A,B,C$ must tokenize to the same width. Therefore we exclude evaluation of the hours domain for the Qwen, Gemma, and Mistral families since all of them use a single digit tokenizer.
\subsection{List of variables}
\label{app:variables}
Tables~\ref{tab:nuisance-variables} and~\ref{tab:cycle-variables} show the values of variables used for replicate variables and $A,B,C$ respectively.

\subsection{Capabilities}
Figure~\ref{fig:app-accuracy-prompts} shows the top-1 and top-3 accuracy of each model indexed by $\gamma$ over prompt templates in Table~\ref{tab:frozen-prompts}. Table~\ref{tab:model-domain-accuracy} lists the aggregate top-1 and top-3 accuracy over replicates.

\begingroup
\normalsize

\setlength{\LTleft}{0pt}
\setlength{\LTright}{0pt}
\setlength{\LTcapwidth}{\linewidth}

\fvset{
fontsize=\normalsize,
breaklines=true,
breakanywhere=true,
breaksymbolleft={},
breaksymbolright={}
}

\begin{longtable}{
@{}
>{\raggedright\arraybackslash}p{1.55in}
@{\hspace{0.08in}}
>{\raggedright\arraybackslash}p{0.65in}
@{\hspace{0.08in}}
>{\raggedright\arraybackslash}p{3.14in}
@{}
}
\caption{Prompt templates used for each model and cyclic domain. \(A,B,C\) and the replicate variables vary as described in Tables~\ref{tab:nuisance-variables} and~\ref{tab:cycle-variables}. \texttt{N/A} indicates domains excluded because the cyclic values do not tokenize to a common width.}%
\label{tab:frozen-prompts}\\

\toprule
\textbf{Model} & \textbf{Domain} & \textbf{Prompt template} \\
\midrule
\endfirsthead

\noalign{\vskip 11pt}
\multicolumn{3}{c}{\tablename\ \thetable{} continued} \\[11pt]
\toprule
\textbf{Model} & \textbf{Domain} & \textbf{Prompt template} \\
\midrule
\endhead

\midrule
\noalign{\vskip 11pt}
\multicolumn{3}{r}{Continued on next page} \\
\endfoot

\bottomrule
\endlastfoot

\multirow[t]{4}{1.55in}{\raggedright Llama-3.2-3B}
& months
& \Verb|Hi {name}, the {noun} is scheduled from {A} {dates} to {B} {dates}. This is the same as time from {C} {dates} to| \\

& hours
& \Verb|Hi {name}, the {noun} is scheduled from {day} {A} o'clock to {day} {B} o'clock. This is the same time as from {day} {C} o'clock to {day} | \\

& weekdays
& \Verb|Hi {name}, the {noun} is scheduled from {A} to {B}. This is the same as time from {C} to| \\

& music
& \Verb|Hi {name}, on the musical scale, the interval from the note {A} to the note {B} is the same as the interval from the note {C} to the note| \\

\addlinespace[0.8em]

\multirow[t]{4}{1.55in}{\raggedright Llama-3.1-8B}
& months
& \Verb|Hi {name}, the {noun} is scheduled between {A} {dates} and {B} {dates}. This is the same time as between {C} {dates} and| \\

& hours
& \Verb|Hi {name}, the {noun} is scheduled between {day} {A} o'clock and {day} {B} o'clock. This is the same time as between {day} {C} o'clock and {day} | \\

& weekdays
& \Verb|Hi {name}, the {noun} is scheduled between {A} and {B}. This is the same as time between {C} and| \\

& music
& \Verb|Hi {name}, on the musical scale, the interval from the note {A} to the note {B} is the same as the interval from the note {C} to the note| \\

\addlinespace[0.8em]

\multirow[t]{4}{1.55in}{\raggedright Qwen-2.5-7B}
& months
& \Verb|Hello {name}, the {noun} is scheduled from {A} {datesDouble} to {B} {datesDouble}. This is the same as duration from {C} {datesDouble} to| \\

& hours
& \Verb|N/A| \\

& weekdays
& \Verb|Hi {name}, the {noun} is scheduled from {A} to {B}. This is the same time as from {C} to| \\

& music
& \Verb|Hi {name}, on the musical scale, the interval from the note {A} to the note {B} is the same as the interval from the note {C} to the note| \\

\addlinespace[0.8em]

\multirow[t]{4}{1.55in}{\raggedright Qwen-3-8B}
& months
& \Verb|Hello {name}, the {noun} is scheduled from {A} {datesDouble} to {B} {datesDouble}. This is the same time as from {C} {datesDouble} to| \\

& hours
& \Verb|N/A| \\

& weekdays
& \Verb|Hi {name}, the {noun} is scheduled from {A} to {B}. This is the same time as from {C} to| \\

& music
& \Verb|Hi {name}, on the musical scale, the interval from the note {A} to the note {B} is the same as the interval from the note {C} to the note| \\

\addlinespace[0.8em]

\multirow[t]{4}{1.55in}{\raggedright Gemma-2-9B}
& months
& \Verb|Hello {name}, the {noun} is scheduled from {A} to {B}. This is the same as duration from {C} to| \\

& hours
& \Verb|N/A| \\

& weekdays
& \Verb|Hi {name}, the {noun} is scheduled from {A} to {B}. This is the same time as from {C} to| \\

& music
& \Verb|Hi {name}, on the musical scale, the interval from the note {A} to the note {B} is the same as the interval from the note {C} to the note| \\

\addlinespace[0.8em]

\multirow[t]{4}{1.55in}{\raggedright Gemma-3-12B}
& months
& \Verb|Hello {name}, the {noun} is scheduled from {A} to {B}. This is the same time as from {C} to| \\

& hours
& \Verb|N/A| \\

& weekdays
& \Verb|Hi {name}, the {noun} is scheduled from {A} to {B}. This is the same time as from {C} to| \\

& music
& \Verb|Hi {name}, on the musical scale, the interval from the note {A} to the note {B} is the same as the interval from the note {C} to the note| \\

\addlinespace[0.8em]

\multirow[t]{4}{1.55in}{\raggedright Mistral-small-24B-Base-2501}
& months
& \Verb|Hi {name}, the {noun} is scheduled from {A} to {B}. This is the same as time from {C} to| \\

& hours
& \Verb|N/A| \\

& weekdays
& \Verb|Hi {name}, the {noun} is scheduled between {A} and {B}. This is the same as time between {C} and| \\

& music
& \Verb|Hi {name}, on the musical scale, the interval from the note {A} to the note {B} is the same as the interval from the note {C} to the note| \\

\end{longtable}
\endgroup

\begin{figure}
    \centering
    \includegraphics[width=\linewidth]{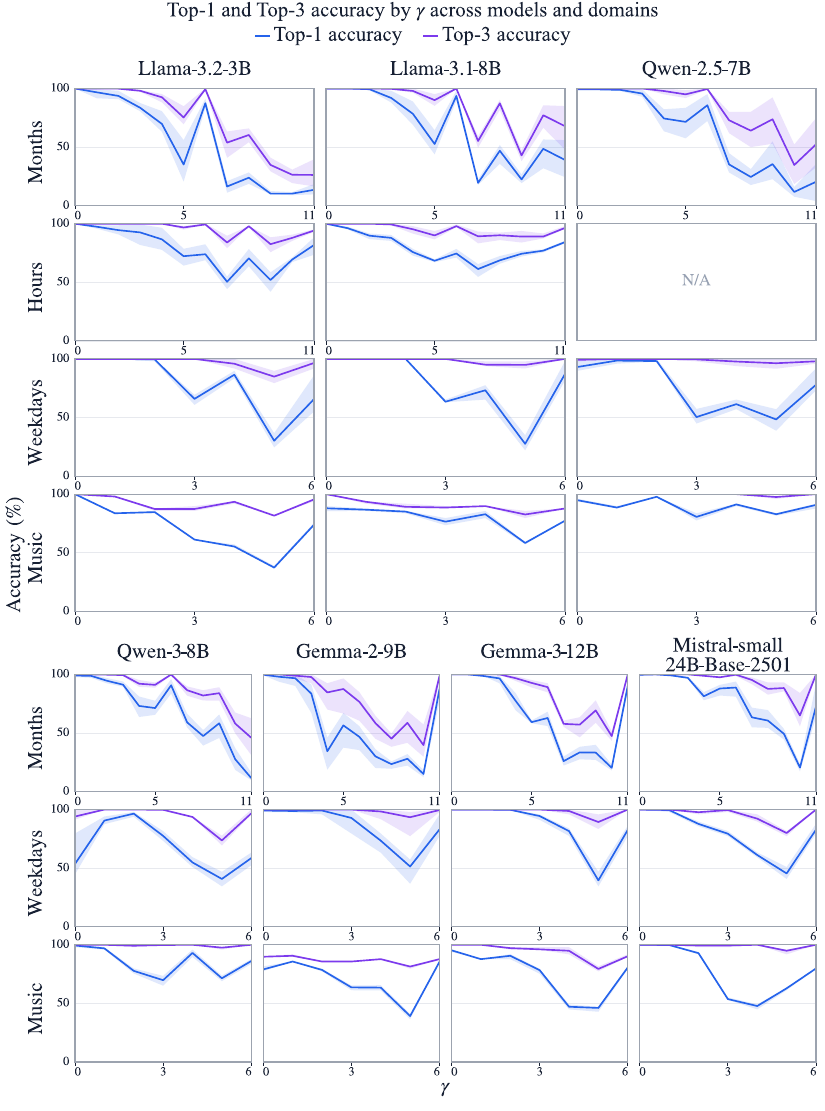}
    \caption{Top-1 and top-3 accuracy as a function of \(\gamma=B-A\) across models and cyclic concepts. Columns correspond to models and rows to concepts. Curves show accuracy averaged across the 10 sampled replicates, with shaded bands showing the central 80\% interval. Hours are omitted for Qwen, Gemma, and Mistral because the relevant numerical values do not tokenize to a common width. Top-3 accuracy is generally more stable across \(\gamma\) than top-1 accuracy.}
    \label{fig:app-accuracy-prompts}
\end{figure}

\section{Theoretical background}\label{app:anova}
This appendix collects properties of the functional ANOVA decomposition used throughout the paper and establishes the identities underlying the mode fraction, the energy fraction and steering analyses. Let us revisit the functional ANOVA decomposition defined in the main text,

\begin{equation*}
    \bm{\phi}^r(A,B,C) = \bm{\mu}^r + \bm{h}_{A}^r + \bm{h}_{B}^r + \bm{h}_{C}^r + \bm{h}_{AB}^r + \bm{h}_{BC}^r + \bm{h}_{CA}^r + \bm{h}_{ABC}^r,
\end{equation*}
where
\begin{equation*}
    \begin{aligned}
    \bm{\mu}^r &:= \mathbb{E}_{ABC}(\bm{\phi}^r),\\
    \bm{h}_{A}^r &:= \mathbb{E}_{BC}(\bm{\phi}^r) - \bm{\mu}^r,\\
    \bm{h}_{AB}^r &:= \mathbb{E}_{C}(\bm{\phi}^r) - \bm{\mu}^r - \bm{h}_{A}^r - \bm{h}_{B}^r,\\
    \bm{h}_{ABC}^r &:= \bm{\phi}^r - \bm{\mu}^r - \bm{h}_{A}^r - \bm{h}_{B}^r - \bm{h}_{C}^r - \bm{h}_{AB}^r - \bm{h}_{BC}^r - \bm{h}_{CA}^r.
\end{aligned}
\end{equation*}
To standardize notation, we use $\bm h^r_{\emptyset} = \bm \mu^r$. Using these definitions, it immediately follows that each non-constant decomposed term has zero mean with respect to each variable on which it depends.
\begin{lemma}[Zero mean of ANOVA terms]
\label{lem:zero-mean}
For every nonempty $S\subseteq\{A,B,C\}$ and every $X\in S$,

\begin{equation}
    \mathbb{E}_X\!\left[\bm h_S^r\right] = \bm 0.
\end{equation}
\end{lemma}
\begin{proof}
    Consider $\bm h_A^r$:
    \begin{align*}
        \mathbb{E}_A(\bm h_A^r) &= \mathbb{E}_A\left( \mathbb{E}_{BC}(\bm{\phi}^r) - \bm{\mu}^r \right)\\
        &= \mathbb{E}_{ABC}(\bm{\phi}^r) - \bm{\mu}^r\\
        &= \bm \mu^r -\bm \mu^r\\
        &=0.
    \end{align*}
    Similarly, for $\bm h_{AB}^r$
    \begin{align*}
        \mathbb{E}_A(\bm h_{AB}^r) &= \mathbb{E}_A\left( \mathbb{E}_{C}(\bm{\phi}^r) - \bm{\mu}^r - \bm{h}_{A}^r - \bm{h}_{B}^r \right)\\
        &= \mathbb{E}_{AC}(\bm{\phi}^r) - \bm{\mu}^r - \mathbb{E}_A(\bm{h}_{A}^r) - \bm{h}_{B}^r\\
        &= \left(\bm \mu^r + \bm h_B^r\right) -\bm \mu^r -0 -\bm h_B^r\\
        &=0.
    \end{align*}
    The same argument holds for $\mathbb{E}_B(\bm h_{AB}) =0$. Finally, consider $\bm h_{ABC}^r$
    \begin{align*}
        \mathbb{E}_A(\bm h_{ABC}^r) &= \mathbb{E}_A\left( \bm{\phi}^r - \bm{\mu}^r - \bm{h}_{A}^r - \bm{h}_{B}^r - \bm{h}_{C}^r - \bm{h}_{AB}^r - \bm{h}_{BC}^r - \bm{h}_{CA}^r \right)\\
        &= \mathbb{E}_{A}(\bm{\phi}^r) - \bm{\mu}^r - \mathbb{E}_A(\bm{h}_{A}^r) - \bm{h}_{B}^r  -\bm{h}_{C}^r - \mathbb{E}_A(\bm{h}_{AB}^r) - \bm{h}_{BC}^r - \mathbb{E}_A(\bm{h}_{CA}^r)\\
        &= \left(\bm \mu^r + \bm h_B^r +\bm{h}_{C}^r + \bm{h}_{BC}^r\right) -\bm \mu^r -0 -\bm h_B^r-\bm{h}_{C}^r - 0 - \bm{h}_{BC}^r - 0\\
        &=0.
    \end{align*}
    Similarly, $\mathbb{E}_B(\bm h_{ABC}^r) = \mathbb{E}_C(\bm h_{ABC}^r)=0$. In general, it can be proved by induction.
\end{proof}

\begin{definition}[Ensemble-averaged form]\label{def:inner-product}
For $S, T \subseteq \{A,B,C\}$ and two decomposed ensembles $H_S^r$ and $H_T^r$ with vectors $\bm h_S^r \in H_S^r$ and $\bm h_T^r \in H_T^r$, define the ensemble-averaged form as
\begin{equation}
\label{eq:ensemble-inner-product}
    \left\langle H_S^r,H_T^r\right\rangle_{\mathrm{ens}}
    :=
    \mathbb E_{ABC}
    \left[
        \left\langle
            \bm h_S^r,\bm h_T^r
        \right\rangle
    \right]
\end{equation}
where the inner product on the RHS is the standard Euclidean inner product defined over $\mathbb{R}^{\text{d}_{\text{model}}}$.
\end{definition}
\begin{table}[t]
\centering
\caption{Replicate variables used in the experiments.}
\label{tab:nuisance-variables}

\vspace{\baselineskip}
\begin{tabularx}{\linewidth}{
@{}
>{\raggedright\arraybackslash}p{1.15in}
>{\raggedright\arraybackslash}X
@{}
}
\toprule
\textbf{Variable}
& \textbf{Values} \\
\midrule

\texttt{name}
&
Each name should tokenize to one token. There are 194 approved names for each Llama, Qwen, and Mistral model,
and 200 for each Gemma model from the list of 200 most common names of the last century~\citep{ssa_babynames_century_2026}. \\

\addlinespace[0.45em]

\texttt{noun}
&
\texttt{conference}, \texttt{meeting},
\texttt{forum}, \texttt{summit} \\

\addlinespace[0.45em]

\texttt{dates}
&
\texttt{1},$\ldots$,\texttt{28} \\

\addlinespace[0.45em]

\texttt{datesDouble}
&
\texttt{10},$\ldots$,\texttt{28}\\

\addlinespace[0.45em]

\texttt{day}
&
\texttt{Monday}, \texttt{Tuesday}, \texttt{Wednesday},
\texttt{Thursday}, \texttt{Friday}, \texttt{Saturday},
\texttt{Sunday} \\

\bottomrule
\end{tabularx}
\end{table}
\begin{table}[t]
\centering
\caption{Values assigned to the cyclic variables $A$, $B$,
and $C$ in each domain. Months and hours have cycle length \(N=12\), while weekdays and musical notes have \(N=7\).}
\label{tab:cycle-variables}

\vspace{\baselineskip}

\begin{tabularx}{\linewidth}{
@{}
>{\raggedright\arraybackslash}p{1.05in}
>{\raggedright\arraybackslash}X
@{}
}
\toprule
\textbf{Domain}
& \textbf{Values of $A$, $B$, and $C$} \\
\midrule

Months
&
\texttt{January}, \texttt{February}, \texttt{March},
\texttt{April}, \texttt{May}, \texttt{June}, \texttt{July},
\texttt{August}, \texttt{September}, \texttt{October},
\texttt{November}, \texttt{December} \\

\addlinespace[0.45em]

Hours
&
\texttt{1}, \texttt{2}, \texttt{3}, \texttt{4}, \texttt{5},
\texttt{6}, \texttt{7}, \texttt{8}, \texttt{9}, \texttt{10},
\texttt{11}, \texttt{12} \\

\addlinespace[0.45em]

Weekdays
&
\texttt{Monday}, \texttt{Tuesday}, \texttt{Wednesday},
\texttt{Thursday}, \texttt{Friday}, \texttt{Saturday},
\texttt{Sunday} \\

\addlinespace[0.45em]

Music
&
\texttt{C}, \texttt{D}, \texttt{E}, \texttt{F},
\texttt{G}, \texttt{A}, \texttt{B} \\
\bottomrule
\end{tabularx}
\end{table}
\begin{table}[t]
\caption{Top-1 and top-3 accuracy by model and domain. Values are mean
\(\pm\) standard deviation across 10 replicates, computed over
all prompts, including coincident values of \(A\), \(B\), and \(C\).}
\label{tab:model-domain-accuracy}
\begin{center}

\begin{tabular}{
@{}
>{\raggedright\arraybackslash}p{1.20in}
@{\hspace{0.07in}}
>{\raggedright\arraybackslash}p{0.85in}
@{\hspace{0.07in}}
>{\centering\arraybackslash}p{0.70in}
@{\hspace{0.07in}}
>{\centering\arraybackslash}p{0.85in}
@{\hspace{0.07in}}
>{\centering\arraybackslash}p{0.80in}
@{\hspace{0.07in}}
>{\centering\arraybackslash}p{0.75in}
@{}
}
\toprule
\textbf{Model}
& \textbf{Performance metric}
& \textbf{Months}
& \textbf{Hours}
& \textbf{Weekdays}
& \textbf{Music} \\
\midrule

\multirow[t]{2}{1.20in}{\raggedright Llama-3.2-3B}
& Top-1 accuracy
& \(53.4 \pm 2.2\)\%
& \(78.5 \pm 3.2\)\%
& \(78.3 \pm 2.5\)\%
& \(70.9 \pm 0.3\)\% \\
& Top-3 accuracy
& \(72.2 \pm 2.8\)\%
& \(95.2 \pm 1.1\)\%
& \(96.8 \pm 0.8\)\%
& \(91.9 \pm 0.5\)\% \\

\addlinespace[0.6em]

\multirow[t]{2}{1.20in}{\raggedright Llama-3.1-8B}
& Top-1 accuracy
& \(66.0 \pm 3.9\)\%
& \(79.9 \pm 1.8\)\%
& \(78.8 \pm 1.1\)\%
& \(79.3 \pm 0.7\)\% \\
& Top-3 accuracy
& \(84.8 \pm 2.6\)\%
& \(94.7 \pm 1.5\)\%
& \(98.6 \pm 0.3\)\%
& \(90.2 \pm 0.4\)\% \\

\addlinespace[0.6em]

\multirow[t]{2}{1.20in}{\raggedright Qwen-2.5-7B}
& Top-1 accuracy
& \(62.7 \pm 7.0\)\%
& N/A
& \(75.4 \pm 2.4\)\%
& \(89.5 \pm 0.6\)\% \\
& Top-3 accuracy
& \(82.5 \pm 6.8\)\%
& N/A
& \(98.7 \pm 0.8\)\%
& \(99.7 \pm 0.2\)\% \\

\addlinespace[0.6em]

\multirow[t]{2}{1.20in}{\raggedright Qwen-3-8B}
& Top-1 accuracy
& \(68.5 \pm 4.5\)\%
& N/A
& \(67.5 \pm 3.4\)\%
& \(84.9 \pm 1.7\)\% \\
& Top-3 accuracy
& \(86.5 \pm 2.6\)\%
& N/A
& \(94.0 \pm 1.0\)\%
& \(99.5 \pm 0.3\)\% \\

\addlinespace[0.6em]

\multirow[t]{2}{1.20in}{\raggedright Gemma-2-9B}
& Top-1 accuracy
& \(58.3 \pm 5.9\)\%
& N/A
& \(85.4 \pm 2.6\)\%
& \(70.7 \pm 0.6\)\% \\
& Top-3 accuracy
& \(78.8 \pm 6.2\)\%
& N/A
& \(98.8 \pm 1.7\)\%
& \(87.0 \pm 0.3\)\% \\

\addlinespace[0.6em]

\multirow[t]{2}{1.20in}{\raggedright Gemma-3-12B}
& Top-1 accuracy
& \(66.5 \pm 1.8\)\%
& N/A
& \(85.5 \pm 1.4\)\%
& \(75.1 \pm 1.3\)\% \\
& Top-3 accuracy
& \(84.1 \pm 1.7\)\%
& N/A
& \(98.3 \pm 1.1\)\%
& \(94.0 \pm 0.8\)\% \\

\addlinespace[0.6em]

\multirow[t]{2}{1.20in}{\raggedright
Mistral-small-24B-Base-2501}
& Top-1 accuracy
& \(76.7 \pm 2.7\)\%
& N/A
& \(79.5 \pm 0.7\)\%
& \(76.7 \pm 0.7\)\% \\
& Top-3 accuracy
& \(94.3 \pm 1.6\)\%
& N/A
& \(95.7 \pm 0.7\)\%
& \(99.1 \pm 0.3\)\% \\

\bottomrule
\end{tabular}

\end{center}
\end{table}
\begin{proposition}[Inner product]
    Under conditions of Definition~\ref{def:inner-product}, Equation~\ref{eq:ensemble-inner-product} satisfies the properties of an inner product:
    \begin{enumerate}
        \item \textbf{Symmetry}. $\left\langle H_S^r,H_T^r\right\rangle_{\mathrm{ens}}=\left\langle H_T^r,H_S^r\right\rangle_{\mathrm{ens}}$.
        \item \textbf{Linearity}. $\left\langle\alpha H_S^r+\beta H_T^r,H_U^r\right\rangle_{\mathrm{ens}}=\alpha\left\langle H_S^r,H_U^r\right\rangle_{\mathrm{ens}}+\beta\left\langle H_T^r,H_U^r\right\rangle_{\mathrm{ens}}$ for any $\alpha,\beta \in \mathbb{R}$.
        \item \textbf{Positive definiteness}. $\left\langle H_S^r,H_S^r\right\rangle_{\mathrm{ens}}\geq 0$ and the equality holds iff $H_S^r = 0$.
    \end{enumerate}
\end{proposition}

\begin{proof}
\textbf{Symmetry.}
By symmetry of the Euclidean inner product,
\[
    \left\langle H_S^r,H_T^r\right\rangle_{\mathrm{ens}}
    =
    \mathbb E_{ABC}
    \left[\left\langle \bm h_S^r,\bm h_T^r\right\rangle\right]
    =
    \mathbb E_{ABC}
    \left[\left\langle \bm h_T^r,\bm h_S^r\right\rangle\right]
    =
    \left\langle H_T^r,H_S^r\right\rangle_{\mathrm{ens}}.
\]

\textbf{Linearity.}
For any $\alpha,\beta\in\mathbb R$ and ensembles
$H_S^r,H_T^r,H_U^r$, linearity of the Euclidean inner product
and of expectation gives
\begin{align*}
    \left\langle
        \alpha H_S^r+\beta H_T^r,H_U^r
    \right\rangle_{\mathrm{ens}}
    &=
    \mathbb E_{ABC}
    \left[
        \left\langle
            \alpha\bm h_S^r+\beta\bm h_T^r,\bm h_U^r
        \right\rangle
    \right] \\
    &=
    \alpha\left\langle H_S^r,H_U^r\right\rangle_{\mathrm{ens}}
    +
    \beta\left\langle H_T^r,H_U^r\right\rangle_{\mathrm{ens}}.
\end{align*}

\textbf{Positive definiteness.}
For every ensemble $H_S^r$,
\[
    \left\langle H_S^r,H_S^r\right\rangle_{\mathrm{ens}}
    =
    \mathbb E_{ABC}\left[\|\bm h_S^r\|_2^2\right]
    \geq 0.
\]
A nonnegative random variable has expectation zero if and only if it is zero. Hence
\[
    \left\langle H_S^r,H_S^r\right\rangle_{\mathrm{ens}}=0
    \quad\Longleftrightarrow\quad
    \bm h_S^r=0.
\]
This is equivalent to $H_S^r=0$ in the ensemble space.

Thus the form satisfies all three inner product axioms.
\end{proof}

\begin{proposition}[Orthogonality of the ANOVA terms]
\label{prop:anova-orthogonality}
Under the full ensemble, any two distinct terms in the decomposition are orthogonal under the inner product in Definition~\ref{def:inner-product}:
\begin{equation}
    \left\langle H_S^r,H_T^r\right\rangle_{\mathrm{ens}}
    =0,
    \qquad
    S\neq T.
\end{equation}
Here, $S,T \subseteq \{A,B,C\}$
\end{proposition}

\begin{proof}
Let $S\neq T$. Without loss of generality, choose a variable $X\in S\setminus T$. Since $\bm h_T^r$ does not depend on $X$, we may average over $X$ first:
\begin{align*}
    \left\langle H_S^r,H_T^r\right\rangle_{\mathrm{ens}}
    &=
    \mathbb E_{\{ABC\}\setminus\{X\}}
    \left[
        \left\langle
            \mathbb E_X\left[\bm h_S^r\right],
            \bm h_T^r
        \right\rangle
    \right] \\
    &=0,
\end{align*}
where the final equality follows from Lemma~\ref{lem:zero-mean}. Hence all distinct terms are pairwise orthogonal.
\end{proof}

\begin{corollary}[Energy decomposition]
\label{cor:anova-energy}
Under the conditions of Proposition~\ref{prop:anova-orthogonality},
\begin{equation}
\label{eq:anova-energy}
    \mathbb{E}_{ABC}
    \left\|
        \bm{\phi}^r(A,B,C)-\bm{h}_{\emptyset}^r
    \right\|_2^2
    =
    \sum_{\emptyset\neq S\subseteq\{A,B,C\}}
    \mathbb{E}_{S}
    \left\|
        \bm h_S^r
    \right\|_2^2.
\end{equation}
\end{corollary}

\begin{proof}
By Equation~\ref{eq:decomp},
\begin{equation*}
    \bm{\phi}^r-\bm{h}_{\emptyset}^r
    =
    \sum_{\emptyset\neq S\subseteq\{A,B,C\}}
    \bm h_S^r.
\end{equation*}
Expanding the squared norm and applying
Proposition~\ref{prop:anova-orthogonality} eliminates all cross terms:
\begin{align*}
    \mathbb{E}_{ABC}
    \left\|
        \bm{\phi}^r-\bm{h}_{\emptyset}^r
    \right\|_2^2
    &=
    \sum_{\emptyset\neq S\subseteq\{A,B,C\}}
    \mathbb{E}_{ABC}
    \left\|
        \bm h_S^r
    \right\|_2^2 \\
    &=
    \sum_{\emptyset\neq S\subseteq\{A,B,C\}}
    \mathbb{E}_{S}
    \left\|
        \bm h_S^r
    \right\|_2^2,
\end{align*}
where the second equality follows because $\bm h_S^r$ depends only on the variables contained in $S$.
\end{proof}

\subsection{Mode fraction and cluster analysis}
\label{app:mode-cluster}
In Section~\ref{sec:fourier} of the main text, we defined mode fraction using the DFT. It can instead be defined directly without invoking the DFT. Here, we show the equivalence. Figure~\ref{fig:decomposition} suggests that the interaction ensembles form clusters indexed by $\gamma$ or $D$.

\begin{definition}[Cluster vectors]
\label{def:cluster-vectors}
For $H_{AB}^r$, we define the centroid of the cluster indexed by
$\gamma$ as
\begin{equation}
\label{eq:cluster-vector-gamma}
    \bar{\bm h}_\gamma^r
    :=
    \mathbb E\!\left[
        \bm h_{AB}^r
        \,\middle|\,
        B-A=\gamma
    \right].
\end{equation}
Similarly, for $H_{ABC}^r$, we define the centroid of the cluster indexed by $D$ as
\begin{equation}
\label{eq:cluster-vector-D}
    \bar{\bm h}_D^r
    :=
    \mathbb E\!\left[
        \bm h_{ABC}^r
        \,\middle|\,
        C+B-A=D
    \right].
\end{equation}
\end{definition}
Note that these are the same vectors used in the steering experiments in Section~\ref{sec:steering}. We now define the corresponding projection operators for each ensemble.
\begin{definition}[Projection operator]
Let $P_\gamma$ be the projection operator that operates on $H_{AB}^r$ to give the corresponding cluster vector:
    \begin{equation}
        P_\gamma H_{AB}^r:=\bar{\bm h}_{\gamma}^r
    \end{equation}
Similarly, we define $P_D$ as
\begin{equation}
    P_D H_{ABC}^r:=\bar{\bm h}_{D}^r
\end{equation}
\end{definition}
Thus, $P_\gamma H_{AB}^r$ is constant over pairs with the same value of $\gamma$, while $P_D H_{ABC}^r$ is constant over triples with the same value of $D$. Now we define cluster fraction using these projection operators and show that it is equivalent to the definition of mode fraction in Section~\ref{sec:fourier}.
\begin{definition}[Cluster fraction]
\label{def:cluster}
Let $c(H^r)$ be the cluster fraction defined by
\begin{align}
\label{eq:cluster-fraction-gamma}
    c(H_{AB}^r)
    :=
    \frac{
        \mathbb E_{\gamma}
        \left\|
            P_\gamma H_{AB}^r
        \right\|_2^2
    }{
        \mathbb E_{AB}
        \left\|
            \bm h_{AB}^r
        \right\|_2^2
    } 
    =
    \frac{
        \mathbb E_{\gamma}
        \left\|
            \bar{\bm h}_{\gamma}^r
        \right\|_2^2
    }{
        \mathbb E_{AB}
        \left\|
            \bm h_{AB}^r
        \right\|_2^2
    }
\end{align}
and
\begin{align}
\label{eq:cluster-fraction-D}
    c(H_{ABC}^r)
    :=
    \frac{
        \mathbb E_{D}
        \left\|
            P_D H_{ABC}^r
        \right\|_2^2
    }{
        \mathbb E_{ABC}
        \left\|
            \bm h_{ABC}^r
        \right\|_2^2
    } 
    =
    \frac{
        \mathbb E_{D}
        \left\|
            \bar{\bm h}_{D}^r
        \right\|_2^2
    }{
        \mathbb E_{ABC}
        \left\|
            \bm h_{ABC}^r
        \right\|_2^2
    }
\end{align}
\end{definition}

\begin{proposition}[Equivalence between cluster fraction and mode fraction]
\label{prop:cluster-mode-equivalence}
If $c(H^r)$ is the cluster fraction as in Definition~\ref{def:cluster} and $m(H^r)$ is the mode fraction as defined in Section~\ref{sec:fourier}, then
\begin{equation}
    c(H_S^r)
    =
    m(H_S^r),
\end{equation}
for $S \in \{AB,BC,CA,ABC\}$.
\end{proposition}

\begin{proof}
We suppress the superscript $r$ for readability.

First consider $H_{AB}$. Let
\begin{equation*}
    \bm p_{AB}
    :=
    P_\gamma H_{AB}
    =
    \bar{\bm h}_\gamma.
\end{equation*}
Here, $\gamma=B-A$, and hence $B=A+\gamma$. The discrete Fourier transform (DFT) of $\bm p$ is
\begin{align*}
    \widehat{\bm p}_{k_Ak_B}
    &=
    \sum_{A,B}
    \bm p_{AB}
    e^{-2\pi i(k_AA+k_BB)/N}\\
    &=
    \sum_{A,\gamma}
    \bar{\bm h}_{\gamma}
    e^{-2\pi i[(k_A+k_B)A+k_B\gamma]/N}\\
    &=
    \left(
        \sum_A
        e^{-2\pi i(k_A+k_B)A/N}
    \right)
    \left(
        \sum_\gamma
        \bar{\bm h}_{\gamma}
        e^{-2\pi i k_B\gamma/N}
    \right).
\end{align*}
Here, $N$ is the length of the cycle; $N=12$ for months and hours, while $N=7$ for weekdays and musical notes. The sum over $A$ vanishes unless $k_A+k_B=0$. Therefore,
\begin{equation*}
    \widehat{\bm p}_{k_Ak_B}=\bm 0
    \qquad
    \text{whenever } k_A\neq-k_B.
\end{equation*}

For frequencies $(k_A,k_B)=(-k,k)$,
\begin{align}
    \widehat{\bm p}_{-k,k}
    &=
    N\sum_\gamma
    \bar{\bm h}_{\gamma}
    e^{-2\pi i k\gamma/N}\\
    &=
    \sum_{A,\gamma}
    \bm h_{A,A+\gamma}
    e^{-2\pi i k\gamma/N}\\
    &=
    \sum_{A,B}
    \bm h_{AB}
    e^{-2\pi i k(B-A)/N}\\
    &=
    \widehat{\bm h}_{-k,k},\label{eq:equiv}
\end{align}
where the second equality uses
\begin{equation*}
    \bar{\bm h}_{\gamma}
    =
    \frac{1}{N}
    \sum_A
    \bm h_{A,A+\gamma}.
\end{equation*}
Thus, the Fourier transform of $P_\gamma H_{AB}$ agrees with that of
$H_{AB}$ on the modes $k_A=-k_B$ and vanishes on all other modes.

Consider $c(H_{AB})$:
\begin{align*}
    c(H_{AB})
    &=
    \frac{
        \mathbb E_{AB}
        \left\|
            \bm p_{AB}
        \right\|_2^2
    }{
        \mathbb E_{AB}
        \left\|
            \bm h_{AB}
        \right\|_2^2
    }\\
    &=
    \frac{
        \sum_{k_A=-k_B}
        \left\|
            \widehat{\bm p}_{k_Ak_B}
        \right\|_2^2
    }{
        \sum_{k_A,k_B}
        \left\|
            \widehat{\bm h}_{k_Ak_B}
        \right\|_2^2
    }\\
    &=
    \frac{
        \sum_{k_A=-k_B}
        \left\|
            \widehat{\bm h}_{k_Ak_B}
        \right\|_2^2
    }{
        \sum_{k_A,k_B}
        \left\|
            \widehat{\bm h}_{k_Ak_B}
        \right\|_2^2
    }\\
    &=
    m(H_{AB}),
\end{align*}
where we get the second equality using Parseval's identity and the third equality using Equation~\ref{eq:equiv}.

A similar proof holds for $c(H_{ABC}) = m(H_{ABC})$.
\end{proof}

\subsection{Steering vector}
\label{app:steering-vector}
Finally, we show the correspondence between the choice of steering vector as in Section~\ref{sec:steering} and the conditional mean of the activation ensemble.
\begin{proposition}\label{prop:steer-counter}
    If $\bar{\bm h}_\gamma^r$ is the vector according to Definition~\ref{def:cluster-vectors}, then
    \begin{equation}
        \bar{\bm h}_\gamma^r = \mathbb{E}_{ABC}\left[ \bm \phi^r(A,B,C) \; \vert \; B-A = \gamma \right] - \bm \mu^r.
    \end{equation}
    Similarly,
    \begin{equation}
        \bar{\bm h}_D^r = \mathbb{E}_{ABC}\left[ \bm \phi^r(A,B,C) \; \vert \; C+B-A = D \right] - \bm \mu^r.
    \end{equation}
\end{proposition}
\begin{proof}
    To see this, start with the conditional average over $\bm \phi^r$:
\begin{align*}
    \mathbb E\!\left[
        \bm\phi^r
        \,\middle|\,
        B-A=\gamma
    \right]
    &=
    \bm\mu^r
    +\mathbb E\!\left[
        \bm h_A^r
        \,\middle|\,
        B-A=\gamma
    \right]
    +\mathbb E\!\left[
        \bm h_B^r
        \,\middle|\,
        B-A=\gamma
    \right] \\
    &\quad
    +\mathbb E\!\left[
        \bm h_C^r
        \,\middle|\,
        B-A=\gamma
    \right]
    +\mathbb E\!\left[
        \bm h_{AB}^r
        \,\middle|\,
        B-A=\gamma
    \right] \\
    &\quad
    +\mathbb E\!\left[
        \bm h_{BC}^r
        \,\middle|\,
        B-A=\gamma
    \right]
    +\mathbb E\!\left[
        \bm h_{CA}^r
        \,\middle|\,
        B-A=\gamma
    \right] \\
    &\quad
    +\mathbb E\!\left[
        \bm h_{ABC}^r
        \,\middle|\,
        B-A=\gamma
    \right].
\end{align*}

Using Lemma~\ref{lem:zero-mean}, we get
\begin{equation*}
    \mathbb E\!\left[
        \bm h_A^r
        \,\middle|\,
        B-A=\gamma
    \right]
    =
    \mathbb E\!\left[
        \bm h_B^r
        \,\middle|\,
        B-A=\gamma
    \right]
    =
    \mathbb E\!\left[
        \bm h_C^r
        \,\middle|\,
        B-A=\gamma
    \right]
    =
    \bm 0,
\end{equation*}
and
\begin{equation*}
    \mathbb E\!\left[
        \bm h_{BC}^r
        \,\middle|\,
        B-A=\gamma
    \right]
    =
    \mathbb E\!\left[
        \bm h_{CA}^r
        \,\middle|\,
        B-A=\gamma
    \right]
    =
    \mathbb E\!\left[
        \bm h_{ABC}^r
        \,\middle|\,
        B-A=\gamma
    \right]
    =
    \bm 0.
\end{equation*}

Hence,
\begin{equation*}
    \mathbb E\!\left[
        \bm\phi^r
        \,\middle|\,
        B-A=\gamma
    \right]
    =
    \bm\mu^r
    +
    \mathbb E\!\left[
        \bm h_{AB}^r
        \,\middle|\,
        B-A=\gamma
    \right].
\end{equation*}
By definition,
\begin{equation*}
    \bar{\bm h}_\gamma^r
    =
    \mathbb E\!\left[
        \bm h_{AB}^r
        \,\middle|\,
        B-A=\gamma
    \right],
\end{equation*}
and therefore
\begin{equation*}
    \bar{\bm h}_\gamma^r
    =
    \mathbb E\!\left[
        \bm\phi^r
        \,\middle|\,
        B-A=\gamma
    \right]
    -
    \bm\mu^r.
\end{equation*}
A similar proof holds for $\bar{\bm h}_D^r$.
\end{proof}

Therefore, we get the following identities using Proposition~\ref{prop:steer-counter}:
\begin{align}
    &\bar{\bm h}_{\gamma+\delta}^r-\bar{\bm h}_{\gamma}^r
    =
    \mathbb E\!\left[
        \bm\phi^r
        \,\middle|\,
        B-A=\gamma+\delta
    \right]
    -
    \mathbb E\!\left[
        \bm\phi^r
        \,\middle|\,
        B-A=\gamma
    \right],\\
    &\bar{\bm h}_{D+\delta}^r-\bar{\bm h}_{D}^r
    =
    \mathbb E\!\left[
        \bm\phi^r
        \,\middle|\,
        C+B-A=D+\delta
    \right]
    -
    \mathbb E\!\left[
        \bm\phi^r
        \,\middle|\,
        C+B-A=D
    \right].
\end{align}
Thus, the steering using $\bar{\bm h}_{\gamma+\delta}^r-\bar{\bm h}_{\gamma}^r$ is equivalent to the difference-of-means steering.

\subsection{Coincident variables}\label{app:coincident-variables}
We always compute the interaction ensembles using the full set of prompts, including cases with $A=B$, $B=C$, or $C=A$. Table~\ref{tab:llama8b-months-replicate0-layer15-norms} reports the mean norm of the interaction vectors grouped by $\gamma=B-A$, $\gamma'=C-B$, and $\gamma''=A-C$ for a representative ensemble using the months prompt template at layer 15 of Llama-3.1-8B. The vectors corresponding to $\gamma=0,\; \gamma'=0,\;\gamma''=0$ have substantially larger norms than the rest of the ensemble and can therefore bias quantitative results toward coincident-variable cases. We exclude such prompts when reporting the results in the main text. Unless stated otherwise, we repeat our analyses without this exclusion in the following appendices. This also serves as a sanity check that our conclusions are not driven by the coincident-variable cases.

\begin{table}[t]
\centering
\caption{Mean norm of each interaction vector categorized by
$\gamma=B-A$, $\gamma'=C-B$, and $\gamma''=A-C$, respectively, for a single replicate of the months prompt template at layer 15 of Llama-3.1-8B. The interaction vectors associated with zero difference have substantially larger norms than the remaining classes. This motivates the exclusion of coincident-variable prompts from the quantitative analyses in the main text.}
\label{tab:llama8b-months-replicate0-layer15-norms}

\vspace{\baselineskip}

\begingroup
\setlength{\tabcolsep}{2pt}

\begin{tabular*}{\linewidth}{
@{\extracolsep{\fill}}
l
*{12}{c}
@{}
}
\toprule
\textbf{Ensemble}
& $\mathbf{0}$ & $\mathbf{1}$ & $\mathbf{2}$ & $\mathbf{3}$
& $\mathbf{4}$ & $\mathbf{5}$ & $\mathbf{6}$ & $\mathbf{7}$
& $\mathbf{8}$ & $\mathbf{9}$ & $\mathbf{10}$ & $\mathbf{11}$ \\
\midrule

$H_{AB}$
& 7.445 & 5.875 & 3.899 & 3.568
& 3.144 & 2.551 & 2.391 & 2.307
& 2.595 & 2.700 & 3.188 & 4.023 \\

\addlinespace[0.4em]

$H_{BC}$
& 4.343 & 1.921 & 0.918 & 0.961
& 0.971 & 1.021 & 0.853 & 0.976
& 1.022 & 1.378 & 1.377 & 1.575 \\

\addlinespace[0.4em]

$H_{CA}$
& 6.448 & 3.013 & 1.311 & 1.457
& 1.478 & 1.624 & 1.476 & 1.487
& 1.263 & 1.220 & 1.589 & 3.189 \\

\bottomrule
\end{tabular*}

\endgroup
\end{table}

\section{Mode fraction}\label{app:mode-fraction}
\begin{figure}
    \centering
    \includegraphics[width=\linewidth]{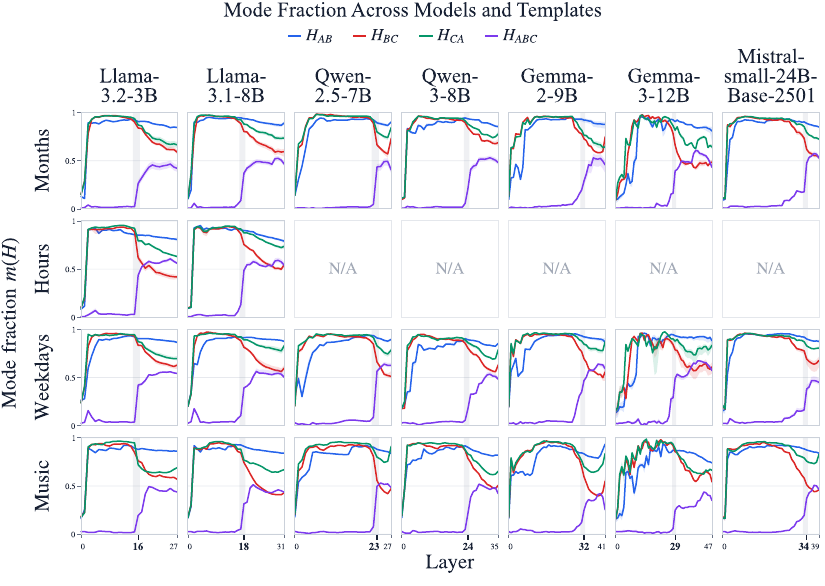}
    \caption{Mode fraction of the interaction ensembles across models and cyclic domains. Columns correspond to models and rows to domains; shaded bands show the central 80\% interval across 10 replicates. The vertical gray band marks the transition in causal relevance from \(H_{AB}^r\) to \(H_{ABC}^r\) identified by the ablation and replacement experiments. Across all models and prompt templates, the pairwise interaction ensembles are strongly organized by their corresponding pairwise differences before \(H_{ABC}^r\) becomes organized by \(D\). The transition in \(H_{ABC}^r\) is more distributed across depth in the Gemma models.}
    \label{fig:app-mode-fraction}
\end{figure}

The main text reports the layerwise organization of the interaction ensembles for Llama-3.1-8B on the months domain. Here, we test whether the same pattern generalizes across models and cyclic domains. Recall the definition of mode fraction for $H_{AB}^r$ as defined in the main text:
\begin{equation}
m(H_{AB}^r) := \frac{\sum_{k_A=-k_B}\norm{\widehat{\bm h}^r_{k_Ak_B}}^2}{\sum_{k_A,k_B}\norm{\widehat{\bm h}^r_{k_Ak_B}}^2}.
\end{equation}

Similarly, we define the corresponding mode fractions for $H_{BC}^r$ and $H_{CA}^r$ using modes satisfying $k_B+k_C =0$ and $k_C+k_A=0$ respectively. For $H_{ABC}^r$, the mode fraction is defined using modes $(k_A, k_B,k_C) = (-k,k,k)$.

Figure~\ref{fig:app-mode-fraction} shows the mode fractions for all interaction ensembles using the prompts and models as detailed in Table~\ref{tab:frozen-prompts}. Across all models and prompt templates, the second-order interaction ensembles become strongly organized by their corresponding pairwise differences before the three-way interaction becomes organized by \(D\). The increase in \(m(H_{ABC}^r)\) occurs near the same relative depth within a particular model across domains. The transition is sharp for most models but is distributed across several layers in the Gemma models. Notice that we report Llama-3.1-8B on the months domain again: we include prompts with coincident variables and we see that the mode fraction is inflated as a result of their inclusion.

\section{Energy Fraction}\label{app:energy-fraction}
\begin{figure}
    \centering
    \includegraphics[width=\linewidth]{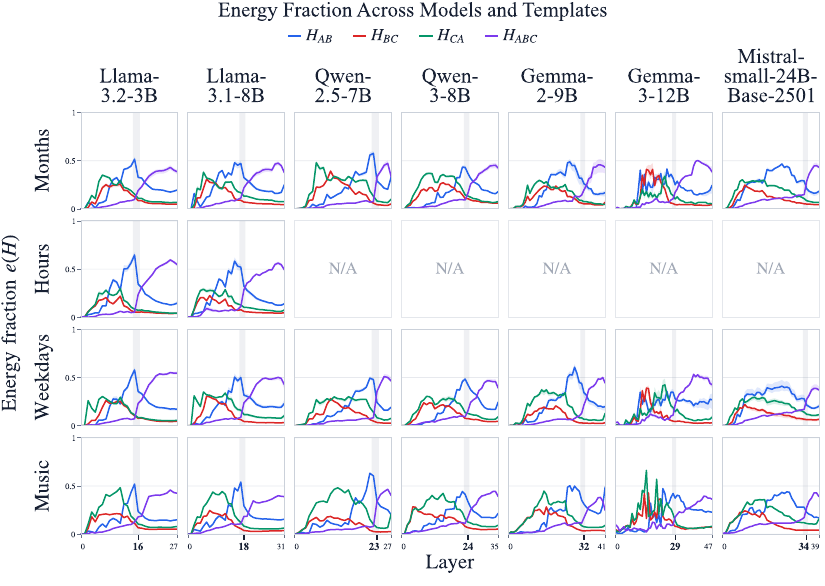}
    \caption{Energy fraction of the interaction ensembles across models and cyclic domains. Columns correspond to models and rows to domains; shaded bands show the central 80\% interval across 10 replicates. The vertical gray band marks the transition in causal relevance identified by the ablation and replacement experiments. Across all models and prompt templates, the energy fraction of \(H_{AB}^r\) decreases near the same depth at which the energy fraction of \(H_{ABC}^r\) increases.}
    \label{fig:app-energy-fraction}
\end{figure}
Recall Corollary~\ref{cor:anova-energy} from Appendix~\ref{app:anova}:
\begin{equation}
\mathbb{E}_{A,B,C}
\left\|
\bm{\phi}^r(A,B,C)-\bm{\mu}^r
\right\|_2^2
=
\sum_{S}
\mathbb{E}_{S}
\left\|
\bm h_S^r
\right\|^2_2,
\end{equation}
where $S$ ranges over non-empty subsets of $\{A,B,C\}$. This motivates us to define the \textit{energy fraction} of each ensemble as
\begin{equation}
    e(H_{S}^r) = \frac{\mathbb{E}_{S}\Big(\norm{\bm{h}^r_{S}}^2\Big)}{\mathbb{E}_{ABC}\Big(\norm{\bm{\phi}^r(A,B,C)-\bm{\mu}^r}^2\Big)}, \qquad \text{with} \qquad \sum_S e(H_S^r) = 1.
\end{equation}
We next test whether the layerwise redistribution of energy fraction observed in the main text generalizes across models and domains. Figure~\ref{fig:app-energy-fraction} shows the energy fraction for each interaction ensemble. Across all models and prompt templates, the energy fraction of \(H_{AB}^r\) decreases near the same depth at which the energy fraction of \(H_{ABC}^r\) increases. As with mode fraction, this transition is spread over several layers in the Gemma models.

\section{Controls for Llama-3.1-8B}
\label{app:controls}
The strong organization of the pairwise interaction ensembles by their corresponding pairwise differences does not by itself imply that this structure is task-specific. We therefore test whether similar organization arises when the cyclic tokens co-occur without the relational computation used in the main task. We design three control prompts as listed in Table~\ref{tab:control-prompts}. For the \emph{non-cyclic} control, we construct a template similar to the Llama-3.1-8B months template but use $A,B,C \in $ \{berry, apple, orange, banana, peach, pear, fig, mango, lemon, lime, cherry, plum\} instead of the usual calendar months. The \emph{off-by-1} control retains \(A,B,C\) as calendar months, but the answer depends only on \(C\); \(A\) and \(B\) therefore co-occur in the prompt without contributing to the answer. The \emph{list} control also retains calendar months but removes the arithmetic relation entirely; \(A,B,C\) are presented as items in a list and the model is asked to return \(B\).

\begin{figure}
    \centering
    \includegraphics[width=\linewidth]{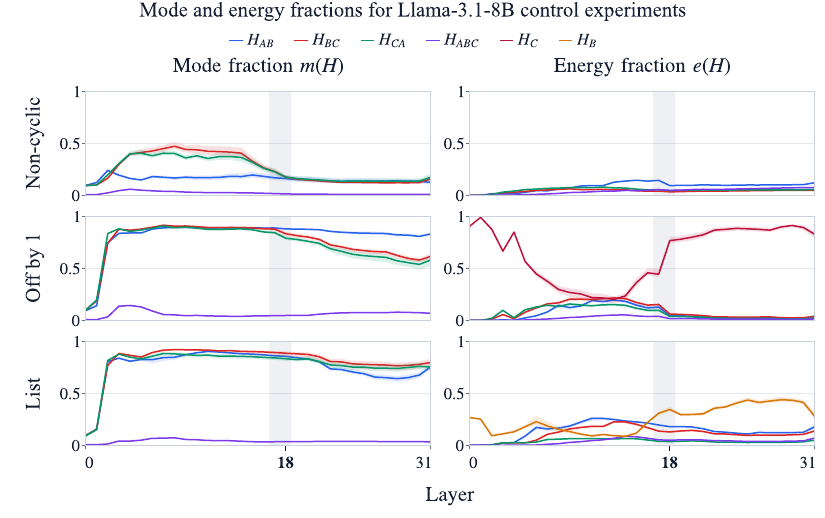}
    \caption{Mode and energy fractions for three control prompt templates using Llama-3.1-8B. The \emph{non-cyclic} control replaces months with an unordered vocabulary, while the \emph{off-by-1} and \emph{list} controls retain calendar months but remove the relational computation of the main task. In the latter two controls, the pairwise interaction ensembles remain strongly organized by their corresponding pairwise differences despite having a small energy fraction. In contrast, \(H_{ABC}^r\) remains negligible in both mode and energy fraction across all controls. All results reported here exclude vectors corresponding to the coincident-variable cases.}
    \label{fig:app-control}
\end{figure}

In Figure~\ref{fig:app-control}, we have computed the mode and energy fractions for the control prompt templates. In all three controls, the mode and energy fractions of \(H_{ABC}^r\) remain close to zero. In contrast, when calendar months co-occur in the \emph{off-by-1} and \emph{list} controls, \(H_{AB}^r\), \(H_{BC}^r\), and \(H_{CA}^r\) can still have large mode fractions, but retain a small energy fraction. This suggests that pairwise-difference geometry can arise from co-occurrence of cyclic tokens alone, whereas the \(D\)-organized interaction ensemble $H_{ABC}^r$ requires the task structure.

\begin{table}[t]
\centering
\normalsize

\caption{Control prompt templates for Llama-3.1-8B.}
\label{tab:control-prompts}
\vspace{\baselineskip}

\begingroup
\setlength{\tabcolsep}{0pt}

\fvset{
  fontsize=\normalsize,
  breaklines=true,
  breakanywhere=true,
  breaksymbolleft={},
  breaksymbolright={}
}

\begin{tabularx}{\linewidth}{
  @{}
  >{\raggedright\arraybackslash}p{0.85in}
  @{\hspace{0.12in}}
  >{\raggedright\arraybackslash}X
  @{}
}
\toprule
\textbf{Control} & \textbf{Prompt template} \\
\midrule

Non-cyclic
&
\Verb|Hello {name}, the {noun} is scheduled between {A} {dates} and {B} {dates}. This is the same time as between {C} {dates} and|
\\[0.6em]

Off by 1
&
\Verb|Hello {name}, the {noun} is scheduled between {A} {dates} and {B} {dates}. The month that comes after {C} {dates} is|
\\[0.6em]

List
&
\Verb|Hello {name}, the talks are on {A} {dates}, on {B} {dates}, and on {C} {dates}. The second talk is on|
\\

\bottomrule
\end{tabularx}
\endgroup
\end{table}

\section{Causal handoff}
\label{app:causal-handoff}
Section~\ref{sec:causal} shows that \(H_{AB}^r\) and \(H_{ABC}^r\) are causally relevant at different depths, but this does not establish whether the later \(H_{ABC}^r\) representation depends on the earlier \(H_{AB}^r\) representation. To test this, we ablate \(H_{AB}^r\) at layer 15 of Llama-3.1-8B on the months domain, allow the model to run normally thereafter, and recompute the ANOVA decomposition at every subsequent layer. We then compare the mode and energy fractions of the resulting \(H_{ABC}^r\) with those of the unmodified model.

Figure~\ref{fig:app-mode-energy-downstream} shows the comparison of mode and energy fractions between the ablated and the unmodified ensemble for $H_{ABC}^r$. In the unmodified case, both the mode fraction and energy fraction of \(H_{ABC}^r\) rise sharply around layer 18. After ablating \(H_{AB}^r\) at layer 15, neither transition occurs at that depth. The mode fraction recovers slightly in later layers, but the energy fraction remains substantially reduced. Thus, the emergence of the later \(H_{ABC}^r\) representation depends on the earlier \(H_{AB}^r\) representation.

\begin{figure}
    \centering
    \includegraphics[width=\linewidth]{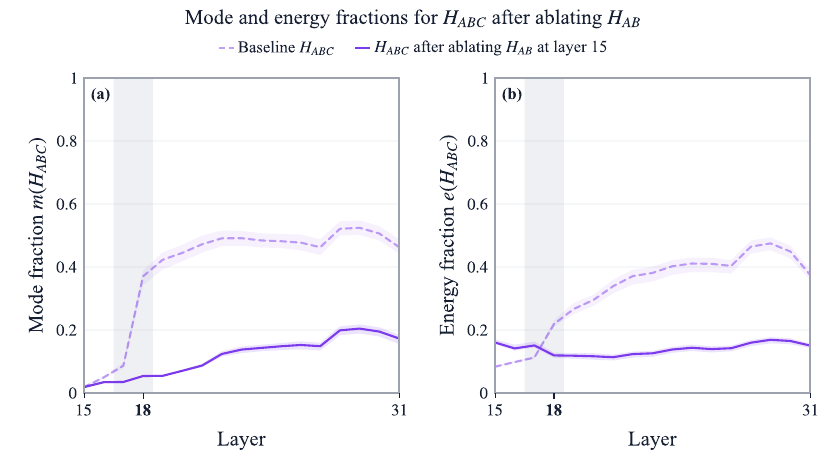}
    \caption{Downstream \(H_{ABC}^r\) after ablating \(H_{AB}^r\) at layer 15. Solid curves show the mode and energy fractions of the recomputed \(H_{ABC}^r\) after the intervention; dashed curves show the corresponding quantities in the unmodified model. The vertical gray band marks layer 18, where \(H_{ABC}^r\) normally becomes organized by \(D\) and increases sharply in energy fraction. \textbf{(a)} After ablating \(H_{AB}^r\), the mode fraction does not undergo its normal increase at layer 18, although some organization recovers later but never regains the unmodified baseline. \textbf{(b)} The energy fraction likewise fails to increase at layer 18 and remains substantially below that of the unmodified ensemble. These results show that the normal emergence of \(H_{ABC}^r\) depends on the earlier \(H_{AB}^r\) representation.}
    \label{fig:app-mode-energy-downstream}
\end{figure}

\section{Ensemble ablation}\label{app:ablation}
We next test whether the layerwise ablation pattern from the main text generalizes across models and cyclic concepts. We conduct the ensemble ablation experiments for all prompts and models in Table~\ref{tab:frozen-prompts}. Recall that we intervene on the last token position at a particular layer using
\begin{equation}
    \widetilde{\bm{\phi}}^r(A,B,C) = \bm \phi^r(A,B,C) - \bm h^r_S,
\end{equation}
for all $ A,B,C \in \mathcal{X}$ and $S \in \{AB,BC,CA, ABC\}$.

Figures~\ref{fig:app-top-1-interaction-ensembles} and~\ref{fig:app-top-3-interaction-ensembles} report the results of the ensemble ablation experiments using top-1 and top-3 accuracy as metrics, respectively. We observe that across every model and domain, the layerwise behavior of interaction ensembles under ablation remains the same. Ablating $H_{AB}^r$ always affects the accuracy at an earlier network depth than ablating $H_{ABC}^r$. However, ablating $H_{BC}^r$ or $H_{CA}^r$ has little to no effect on performance. Ablating $H_{ABC}^r$ at later layers decreases the performance dramatically, with performance dropping to near chance when it is ablated near the final layers.

\section{Ensemble replacement}\label{app:replacement}
We repeat the ensemble replacement experiments across all models and prompt templates using the following intervention
\begin{equation}
    \widetilde{\bm \phi}^r(A,B,C) = \bm \mu^r+\bm h^r_S,
\end{equation}
for all $ A,B,C \in \mathcal{X}$ and $S \in \{AB,BC,CA,ABC\}$.

Figures~\ref{fig:app-top-1-interaction-replacements} and~\ref{fig:app-top-3-interaction-replacements} show the corresponding results using top-1 and top-3 accuracy respectively. Retaining \(H_{BC}^r\) or \(H_{CA}^r\) together with the mean performs similarly to the \(\mu^r\)-only baseline. Retaining \(H_{AB}^r\) preserves performance in intermediate layers but becomes insufficient later in depth, while retaining \(H_{ABC}^r\) becomes effective at later layers and can exceed the unmodified baseline.

Together with Appendix~\ref{app:ablation}, these results show that the transition from \(H_{AB}^r\) to \(H_{ABC}^r\) is reproduced by both ablation and replacement interventions across models and prompt templates.

\section{Steering}\label{app:steering}
We report steering results across all models and prompt templates using the averaged vectors introduced in Section~\ref{sec:steering}. For $S\in\{\gamma,D\}$, we intervene according to
\begin{equation}
    \bm{\widetilde \phi}^{r}(A,B,C)
    =
    \bm\phi^r(A,B,C)
    +
    \alpha
    \left(
        \bm{\bar h}^r_{S+\delta}
        -
        \bm{\bar h}^r_S
    \right).
\end{equation}

To compare steering effects across models, domains, and layers, we report the $\mathrm{normalized\; median\;} \Delta m_3$. We normalize by
\begin{equation}
    \Delta m_3^{\mathrm{max}}
    =
    \mathop{\mathrm{max}}_{S, \delta,r,\ell}
    \left|
        \Delta m_3^{r,\ell}(\bm{\bar h}^r_{S+\delta} - \bm{\bar h}^r_S)
    \right|.
\end{equation}
Figure~\ref{fig:app-steering-pooled} shows the resulting steering profiles at $\alpha=1$. Steering with $\bm{\bar h}_\gamma^r$ is effective in intermediate layers, whereas steering with $\bm{\bar h}_D^r$ becomes effective at later layers. These regions coincide with the layers where $H_{AB}^r$ and $H_{ABC}^r$ are found to be causally relevant in Appendices~\ref{app:ablation} and~\ref{app:replacement}.

\subsection{Sensitivity to steering strength}

The experiments in the main text use $\alpha=1$. To test whether the observed steering effects depend strongly on this choice, Figure~\ref{fig:app-alpha-steering} shows the $\mathrm{median\;} \Delta m_3$ as a function of $\alpha$ at representative layers for $\bm{\bar h}_\gamma^r$ and $\bm{\bar h}_D^r$. In both cases, the steering effect varies smoothly with $\alpha$.

\begin{figure}
    \centering
    \includegraphics[width=\linewidth]{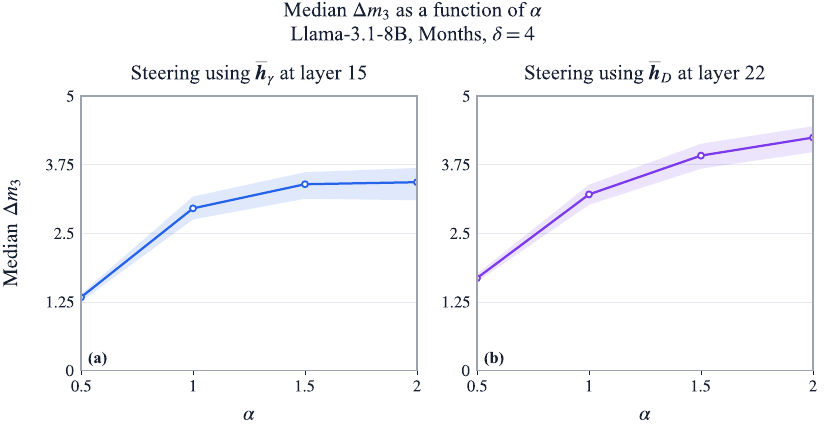}
    \caption{Median $\Delta m_3$ as a function of steering strength $\alpha$ for the averaged steering vectors. \textbf{(a)} Steering with $\bm{\bar h}_\gamma^r$ at layer 15. \textbf{(b)} Steering with $\bm{\bar h}_D^r$ at layer 22. Both panels use the months prompt template with Llama-3.1-8B.}
    \label{fig:app-alpha-steering}
\end{figure}

\subsection{Symmetrized pointwise steering}

The averaged vectors above depend only on the relation label $\gamma$ or $D$. As a complementary intervention, we construct prompt-specific counterfactual interaction vectors by averaging over the coordinate changes that produce the same shift in the corresponding relation.

For $H_{AB}^r$, both $B\to B+\delta$ and $A\to A-\delta$ change $\gamma\to\gamma+\delta$ and therefore $D\to D+\delta$. We define
\begin{equation}
    \widetilde{\bm h}^r_{AB\mid\delta}
    =
    \frac{1}{2}
    \left(
        \bm h^r_{A,B+\delta}
        +
        \bm h^r_{A-\delta,B}
    \right).
\end{equation}
Similarly, each of $A\to A-\delta$, $B\to B+\delta$, and $C\to C+\delta$ changes $D$ to $D+\delta$, so we define
\begin{equation}
    \widetilde{\bm h}^r_{ABC\mid\delta}
    =
    \frac{1}{3}
    \left(
        \bm h^r_{A-\delta,B,C}
        +
        \bm h^r_{A,B+\delta,C}
        +
        \bm h^r_{A,B,C+\delta}
    \right).
\end{equation}
When $\delta=0$, these expressions reduce to the original interaction vectors. We steer using
\begin{equation}
    \bm{\widetilde \phi}^{r}(A,B,C)
    =
    \bm\phi^r(A,B,C)
    +
    \alpha
    \left(
        \widetilde{\bm h}^r_{S\mid\delta}
        -
        \bm h^r_S
    \right),
\end{equation}
where $S\in\{AB,ABC\}$.

Figure~\ref{fig:app-steering-symmetrized} shows the $\mathrm{normalized\; median\;} \Delta m_3$ obtained using this intervention. The same layerwise pattern appears as with the averaged steering vectors: steering using $\widetilde{\bm h}^r_{AB\mid\delta}$ is effective in intermediate layers, whereas steering using $\widetilde{\bm h}^r_{ABC\mid\delta}$ becomes effective at later layers.

\section{Cross-domain transfer}\label{app:cross-domain}

% \subsection{Direct transfer}

We report the cross-domain transfer experiments across models and domains. Figures~\ref{fig:app-top-1-cross-domain-ablations} and~\ref{fig:app-top-3-cross-domain-ablations} show top-1 and top-3 accuracy after transplanting $H_{AB}^r$ from one domain into another following ablation of the corresponding domain-1 interaction. Figures~\ref{fig:app-top-1-cross-domain-replacement} and~\ref{fig:app-top-3-cross-domain-replacement} show the corresponding replacement experiments, in which the domain-1 mean is retained together with the domain-2 interaction. Figure~\ref{fig:app-cross-domain-steering} shows cross-domain steering using $\bar{\bm h}_\gamma$, including $\delta=0$ as a control.

Direct transplantation of $H_{AB}^r$ generally recovers part of the performance lost under ablation, although the degree of recovery depends on the source and target domains. In particular, transfer between the months and hours domains is asymmetric for the Llama models: months$\to$hours decreases performance relative to the unmodified model, whereas hours$\to$months improves the performance (Figures~\ref{fig:app-top-1-cross-domain-ablations} and~\ref{fig:app-top-3-cross-domain-ablations}). In contrast, direct transfer of $H_{ABC}^r$ is substantially less consistent across domains (Figures~\ref{fig:app-top-3-cross-domain-habc-ablation} and~\ref{fig:app-top-3-cross-domain-habc-replacement}).

\subsection{Alignment of interaction subspaces}

The inconsistent direct transfer of $H_{ABC}^r$ may arise because different domains use different residual-stream subspaces for the same underlying interaction structure. We therefore align the interaction subspaces before transplantation. We perform the alignment in a $q=N-1$ dimensional subspace and divide the replicates into a training set (70\%), denoted by $r_t$, and a held-out test set (30\%), denoted by $r_e$.

For domain $c\in\{1,2\}$, let
\begin{equation}
W_{AB}^{(c)}
=
\left[\bm h_{AB}^{r_t,(c)}\right]
\in
\mathbb{R}^{tN^2\times d_{\mathrm{model}}},
\end{equation}
and similarly,
\begin{equation}
W_{ABC}^{(c)}
=
\left[\bm h_{ABC}^{r_t,(c)}\right]
\in
\mathbb{R}^{tN^3\times d_{\mathrm{model}}}.
\end{equation}
Both $W_{AB}^{(c)}$ and $W_{ABC}^{(c)}$ have zero mean by Lemma~\ref{lem:zero-mean}.

We describe the alignment procedure for $W_{AB}$; the same procedure is applied to $W_{ABC}$. Consider transplanting the domain-2 interaction into domain-1. For each domain, we first compute
\begin{equation}
W_{AB}^{(c)}
=
U^{(c)}
\Sigma^{(c)}
\left(V^{(c)}\right)^T,
\end{equation}
where the columns of $V^{(c)}$ are the right singular vectors. Let
$V_q^{(c)}\in\mathbb{R}^{d_{\mathrm{model}}\times q}$ contain the leading
$q=N-1$ right singular vectors. We project the interaction ensembles into these subspaces:
\begin{equation}
Z_{AB}^{(c)}
=
W_{AB}^{(c)}V_q^{(c)}
\in
\mathbb{R}^{tN^2\times q}.
\end{equation}

We then align the domain-2 coordinates to domain-1 by solving the orthogonal Procrustes problem
\begin{equation}
R_{AB}^*
=
\underset{R^TR=\mI_q}{\operatorname{arg\;min}}
\left\lVert
Z_{AB}^{(2)}R-Z_{AB}^{(1)}
\right\rVert_F^2,
\end{equation}
where $R_{AB}^*\in\mathbb{R}^{q\times q}$ is the orthogonal transformation that best aligns the two coordinate systems.

Transforming a domain-2 interaction vector into the domain-1 residual-stream basis then consists of projection onto the domain-2 subspace, alignment in the $q$-dimensional coordinate system, and reconstruction in the domain-1 basis. We therefore define
\begin{equation}
Q
=
V_q^{(2)}
R_{AB}^*
\left(V_q^{(1)}\right)^T
\in
\mathbb{R}^{d_{\mathrm{model}}\times d_{\mathrm{model}}}.
\end{equation}

We evaluate the learned alignment only on held-out replicates. For the aligned ablation-transplant experiment, we use
\begin{equation}
\label{eq:aligned-ablation}
\widetilde{\bm{\phi}}^{r_e,(1)}(A,B,C)
=
\bm{\phi}^{r_e,(1)}(A,B,C)
-
\bm h_{AB}^{r_e,(1)}
+
\bm h_{AB}^{r_e,(2)}Q.
\end{equation}
For the corresponding aligned replacement experiment, we use
\begin{equation}
\widetilde{\bm{\phi}}^{r_e,(1)}(A,B,C)
=
\bm{\mu}^{r_e,(1)}
+
\bm h_{AB}^{r_e,(2)}Q.
\end{equation}

Figures~\ref{fig:app-top-1-aligned-ablation} and~\ref{fig:app-top-3-aligned-ablation} show the aligned ablation-transplant results for both $H_{AB}^r$ and $H_{ABC}^r$, while Figures~\ref{fig:app-top-1-aligned-replacement} and~\ref{fig:app-top-3-aligned-replacement} show the corresponding replacement experiments. Alignment improves cross-domain transfer for both interaction ensembles and produces a particularly large improvement for $H_{ABC}^r$, whose direct transfer is otherwise inconsistent across domains. This suggests that part of the failure of direct $H_{ABC}^r$ transplantation can be explained by differences in the residual-stream subspaces used by different domains.
% Ablation figures
\begin{figure}
    \centering
    \includegraphics[width=\linewidth]{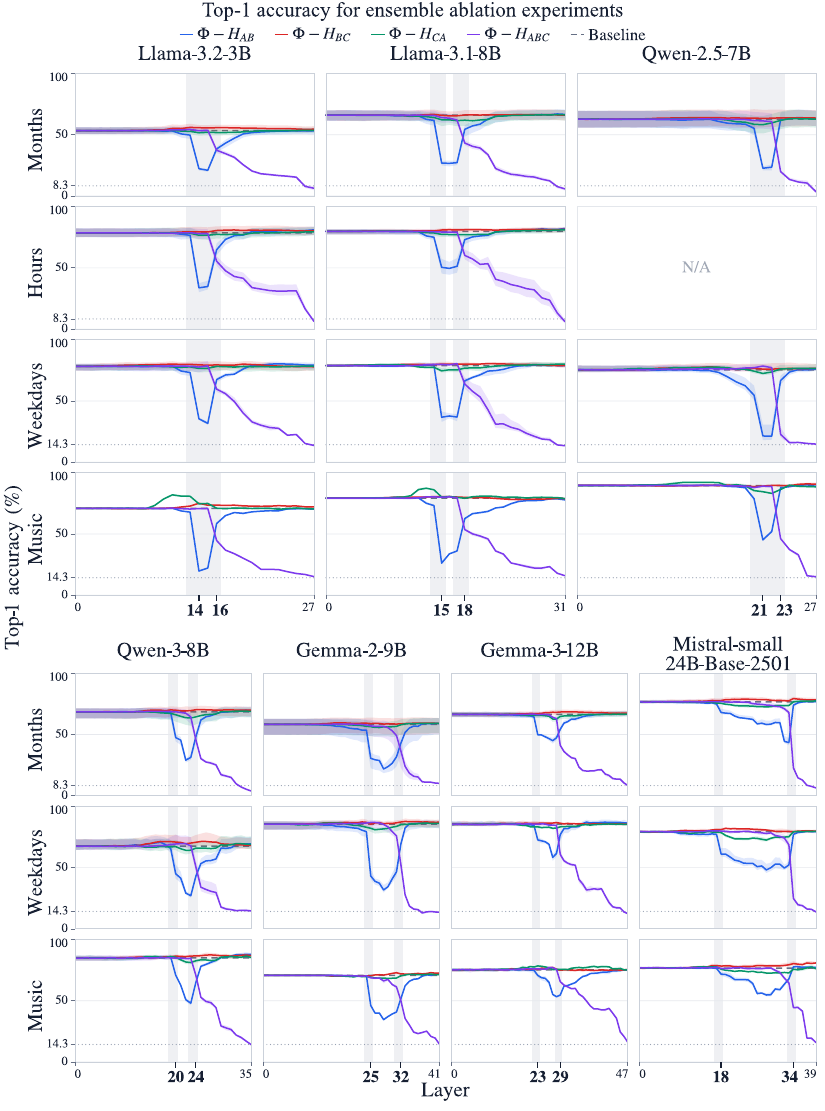}
    \caption{Top-1 accuracy after ablating an interaction ensemble at the last-token position, shown as a function of layer. The vertical shaded bands mark the model-specific intermediate layers where ablating $H_{AB}^r$ produces its largest drop in performance. Across models and prompts, ablating $H_{BC}^r$ or $H_{CA}^r$ has little effect on accuracy, whereas ablating $H_{AB}^r$ or $H_{ABC}^r$ substantially degrades performance. The effect of ablating $H_{AB}^r$ occurs earlier in depth than the effect of ablating $H_{ABC}^r$.}
    \label{fig:app-top-1-interaction-ensembles}
\end{figure}

\begin{figure}
    \centering
    \includegraphics[width=\linewidth]{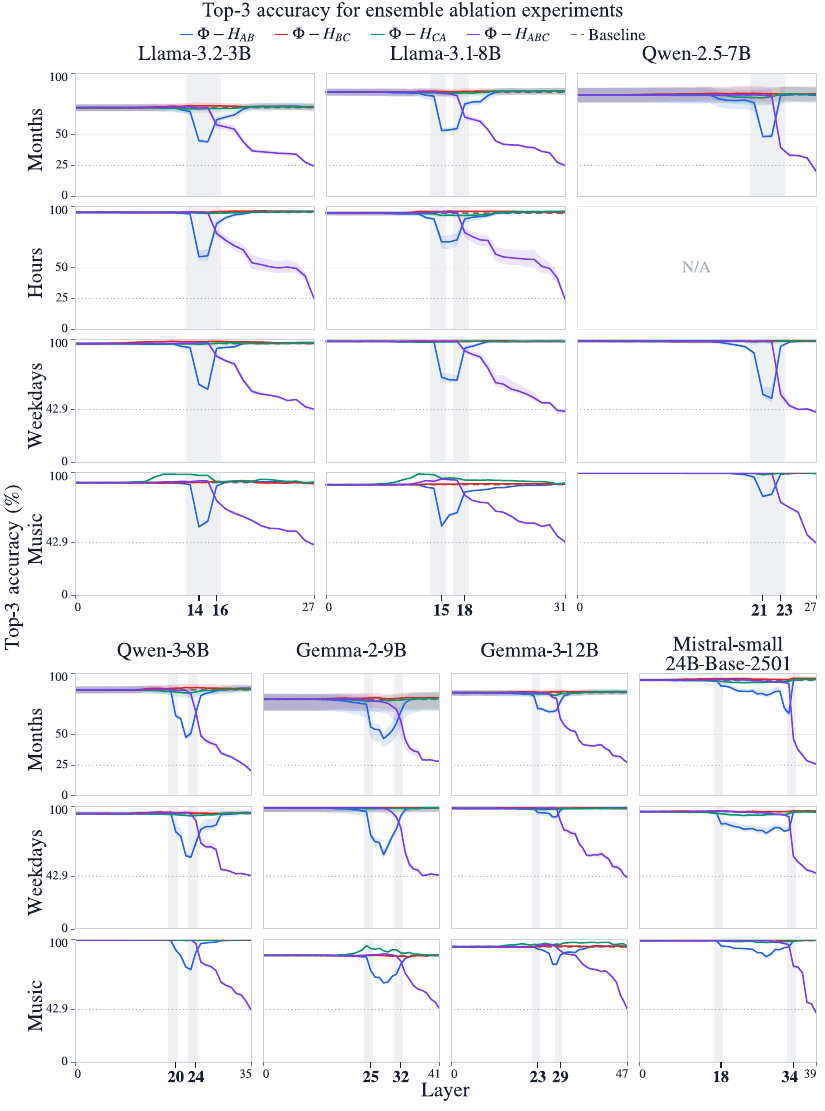}
    \caption{Top-3 accuracy after ablating each interaction ensemble at the last-token position, shown as a function of layer. Across models and prompt templates, ablating \(H_{BC}^r\) or \(H_{CA}^r\) has little effect on performance, whereas ablating \(H_{AB}^r\) or \(H_{ABC}^r\) substantially decreases accuracy. The effect of ablating \(H_{AB}^r\) consistently occurs earlier in depth than the effect of ablating \(H_{ABC}^r\).}
    \label{fig:app-top-3-interaction-ensembles}
\end{figure}

% Replacement figures
\begin{figure}
    \centering
    \includegraphics[width=\linewidth]{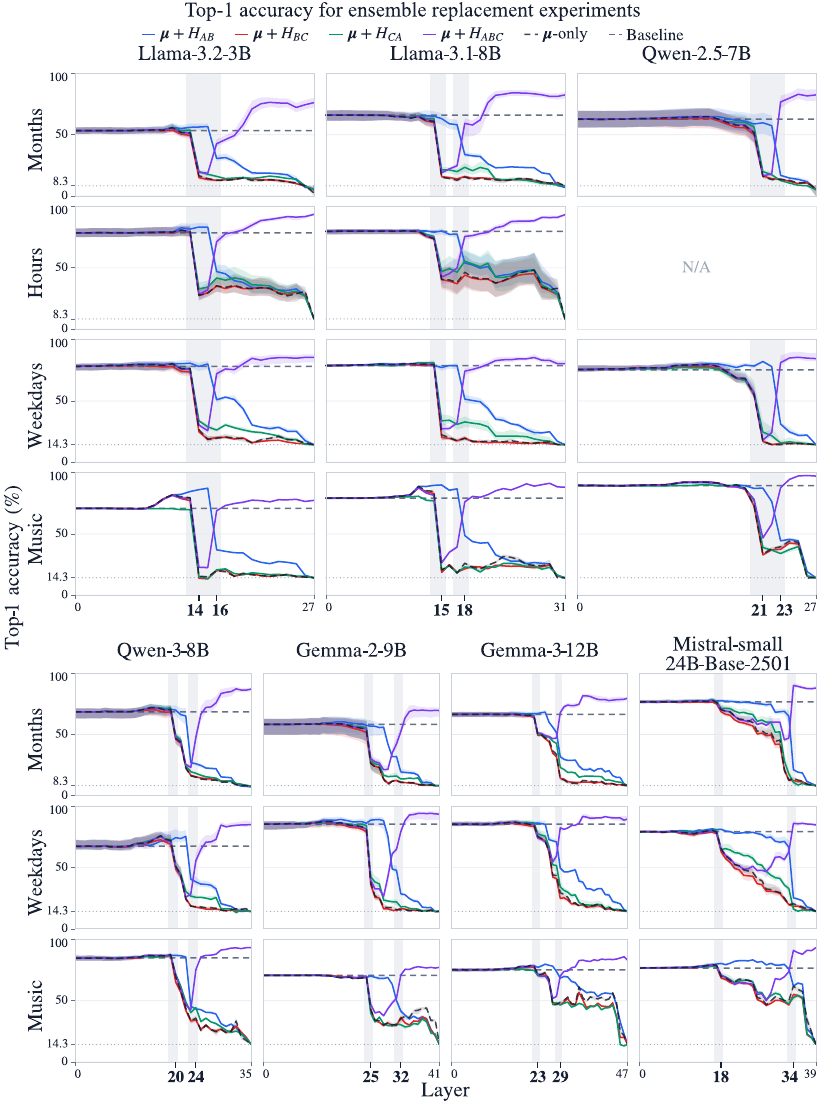}
    \caption{Top-1 accuracy after replacing the activation of the last token with only the mean $\bm \mu^r$ and an interaction ensemble, shown as a function of layer. The vertical shaded bands mark the intermediate layers where ablating $H_{AB}^r$ produced the largest drop in performance, as shown in Figures~\ref{fig:app-top-1-interaction-ensembles} and~\ref{fig:app-top-3-interaction-ensembles}. Replacing with $H_{BC}^r$, $H_{CA}^r$, or $H_{ABC}^r$ substantially degrades performance at the first vertical band, whereas replacing with $H_{AB}^r$ does not affect performance until the second vertical band. At later layers, retaining \(H_{ABC}^r\) preserves performance and can exceed the unmodified baseline.}
    \label{fig:app-top-1-interaction-replacements}
\end{figure}

\begin{figure}
    \centering
    \includegraphics[width=\linewidth]{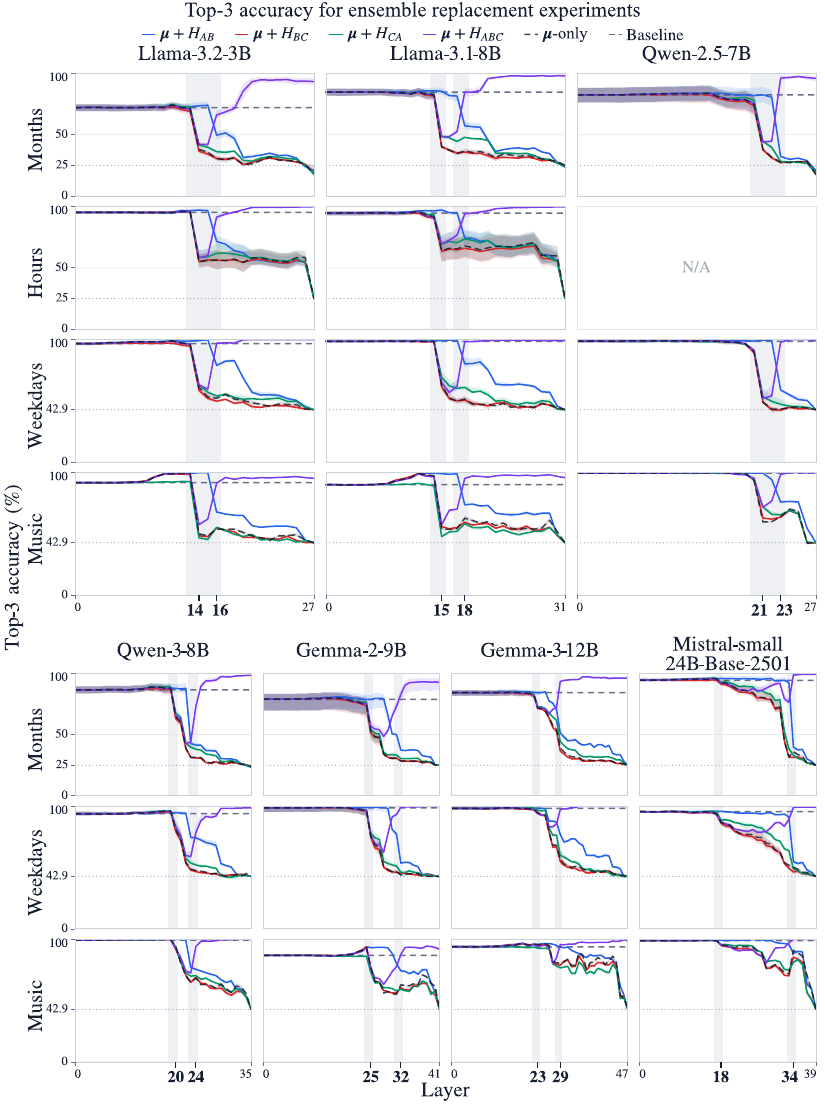}
    \caption{Top-3 accuracy after replacing the last-token activation with the ensemble mean \(\bm \mu^r\) and one interaction term. Across models and prompt templates, retaining \(H_{AB}^r\) preserves performance in intermediate layers but becomes insufficient later in depth, while retaining \(H_{ABC}^r\) becomes effective at later layers and can outperform the unmodified baseline. Retaining \(H_{BC}^r\) or \(H_{CA}^r\) generally degrades the performance beginning from the first vertical band.}
    \label{fig:app-top-3-interaction-replacements}
\end{figure}

% Steering figures
\begin{figure}
    \centering
    \includegraphics[width=\linewidth]{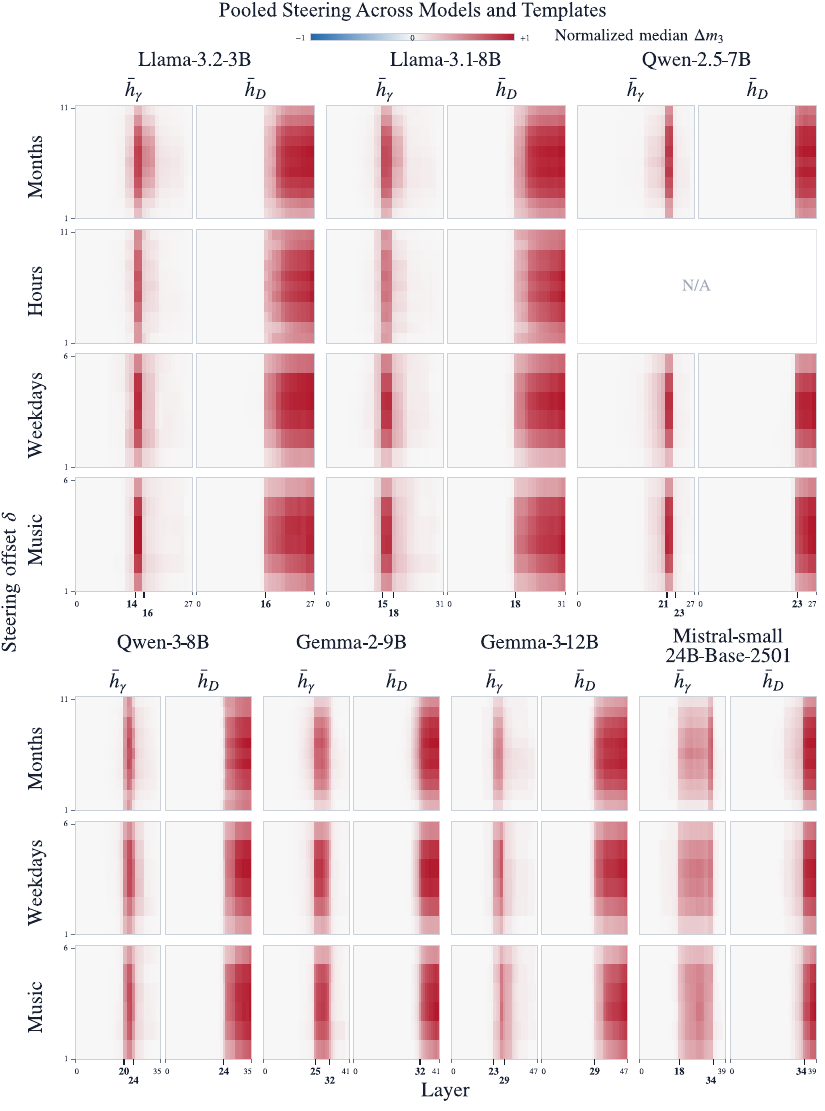}
    \caption{Heatmap of steering using $\bar{\bm h}_\gamma$ and $\bar{\bm h}_D$ across every model and prompt. Color indicates $\mathrm{normalized\;median\;}\Delta m_3$. Steering with $\bar{\bm h}_\gamma$ is most effective in intermediate layers, with boldface layer labels marking the layers where $H_{AB}^r$ was found to be causally relevant. At later layers, steering with $\bar{\bm h}_D$ becomes effective.}
    \label{fig:app-steering-pooled}
\end{figure}

\begin{figure}
    \centering
    \includegraphics[width=\linewidth]{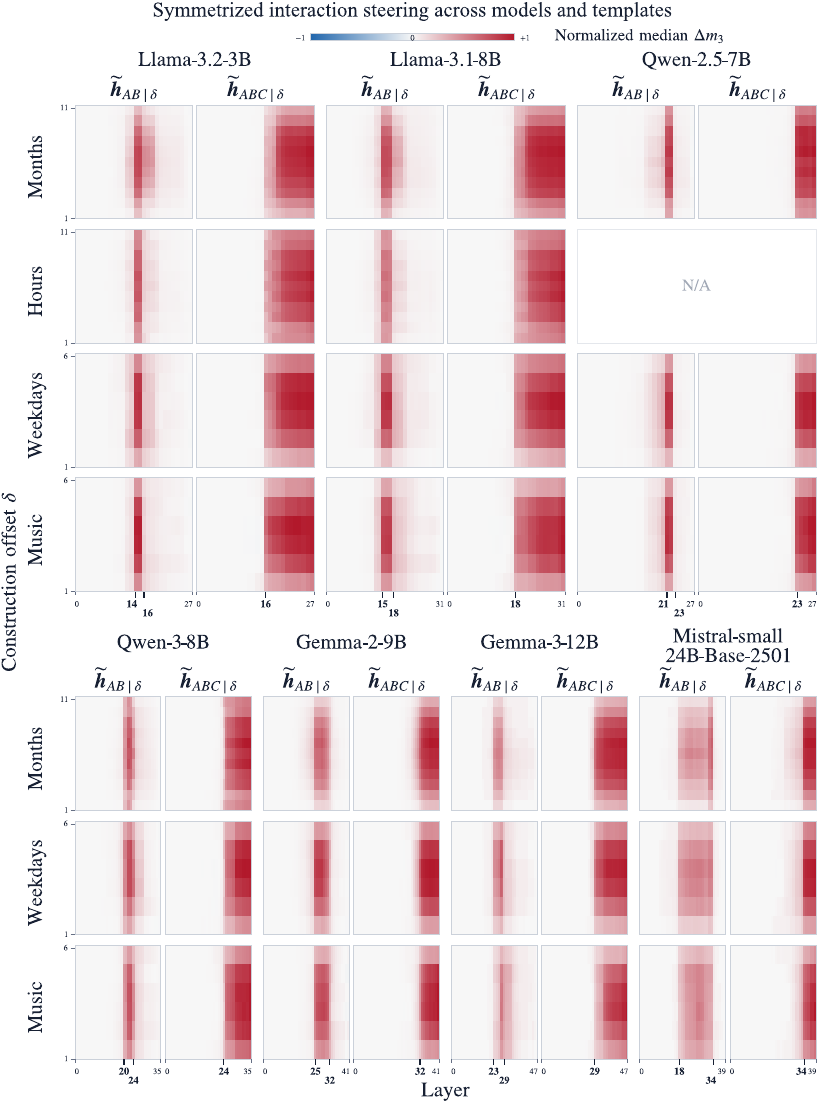}
    \caption{Heatmap of steering using $\widetilde{\bm h}_{AB|\delta}$ and $\widetilde{\bm h}_{ABC|\delta}$ across every model and prompt. Color indicates $\mathrm{normalized\;median\;}\Delta m_3$. As with the steering results in Figure~\ref{fig:app-steering-pooled}, steering using $\widetilde{\bm h}_{AB|\delta}$ is effective in intermediate layers, while steering using $\widetilde{\bm h}_{ABC|\delta}$ becomes effective at later layers.}
    \label{fig:app-steering-symmetrized}
\end{figure}

\begin{figure}
    \centering
    \includegraphics[width=\linewidth]{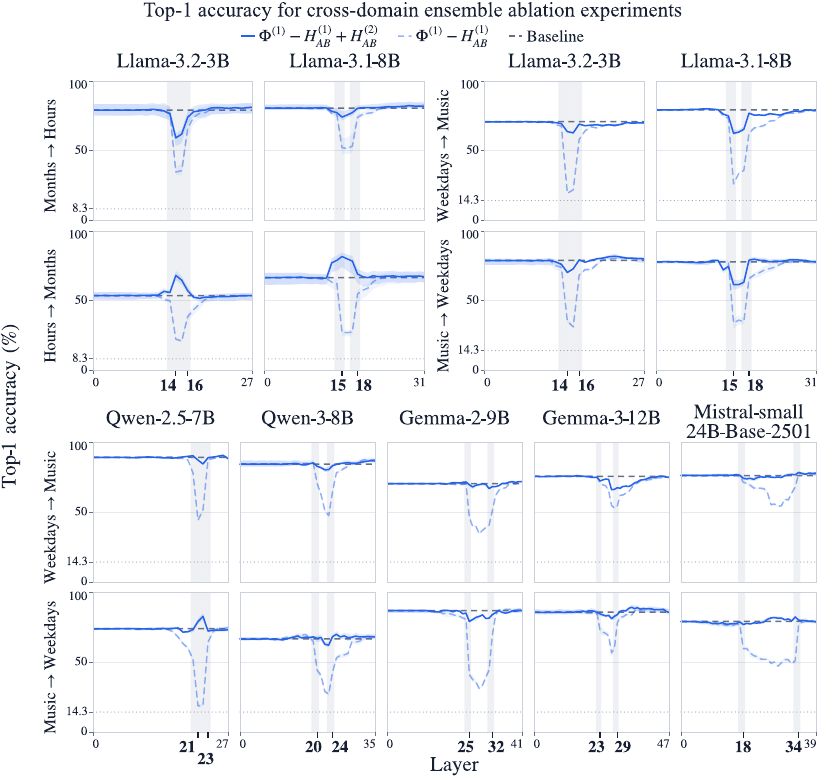}
    \caption{Top-1 accuracy for cross-domain ablation transplant of \(H_{AB}^r\). Solid curves show performance after replacing an ablated domain-1 interaction vector with the corresponding domain-2 vector; dashed blue curves show ablation without transplantation and dashed black curves show the unmodified baseline. Successful transfer is indicated by recovery from the dashed blue ablation curve toward the baseline. Direct transplantation generally recovers performance, although the amount of recovery depends on the model and cross-domain pair.}
    \label{fig:app-top-1-cross-domain-ablations}
\end{figure}

\begin{figure}
    \centering
    \includegraphics[width=\linewidth]{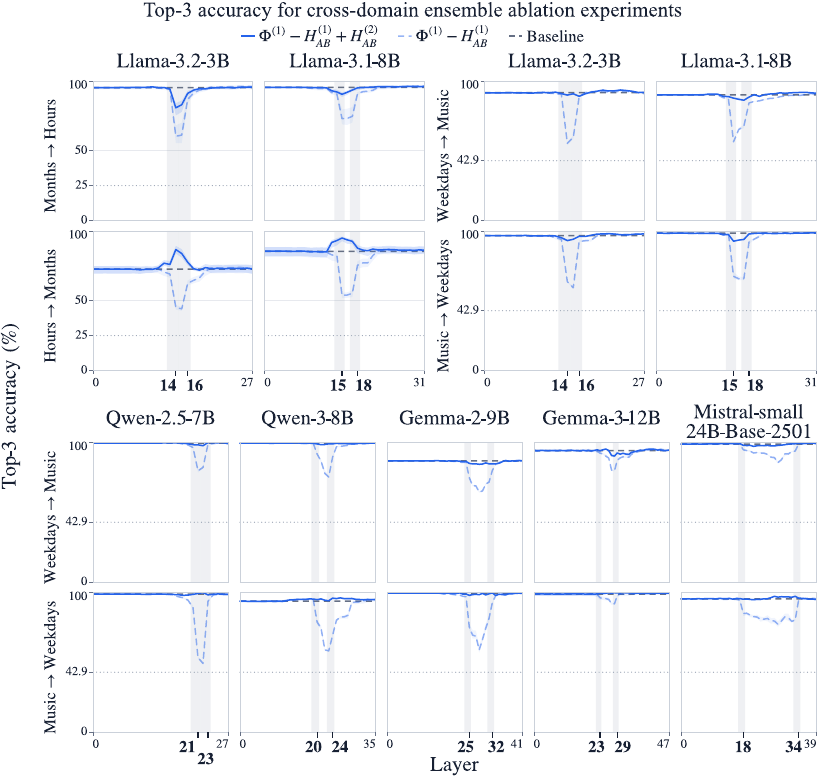}
    \caption{Top-3 accuracy for direct cross-domain ablation transplant of \(H_{AB}^r\). As in Figure~\ref{fig:app-top-1-cross-domain-ablations}, transplantation generally recovers part of the performance lost under ablation, with variation across model and cross-domain pairs.}
    \label{fig:app-top-3-cross-domain-ablations}
\end{figure}

\begin{figure}
    \centering
    \includegraphics[width=\linewidth]{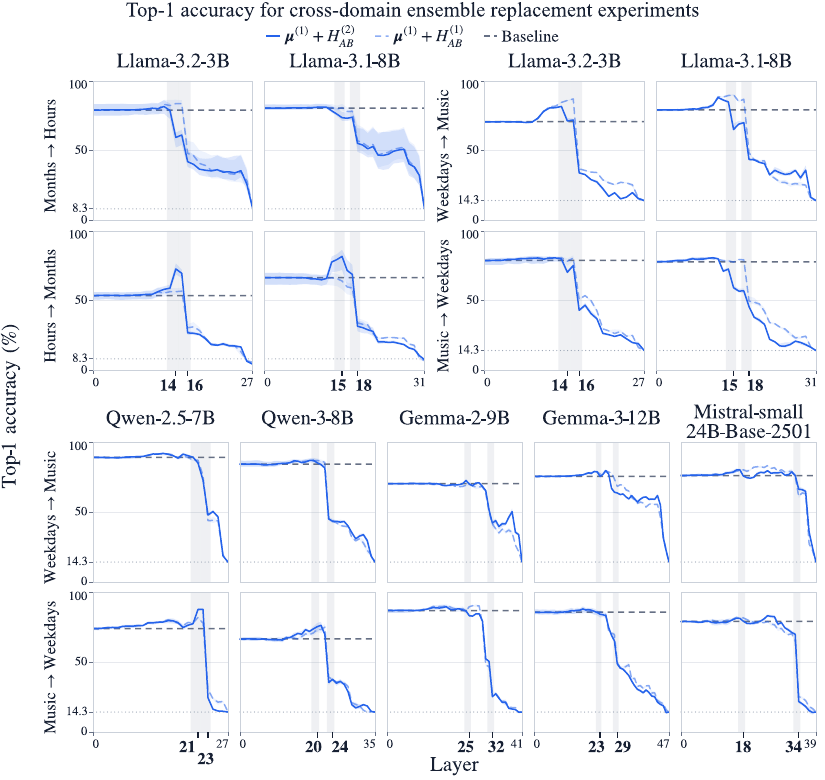}
    \caption{Top-1 accuracy for cross-domain replacement using \(H_{AB}^r\). Solid curves retain the domain-1 mean together with the domain-2 interaction vector; dashed blue curves show the corresponding within-domain replacement experiment. Solid and dashed trajectories follow each other, which indicates that the domain-2 interaction can substitute for the domain-1 interaction in downstream computation.}
    \label{fig:app-top-1-cross-domain-replacement}
\end{figure}

\begin{figure}
    \centering
    \includegraphics[width=\linewidth]{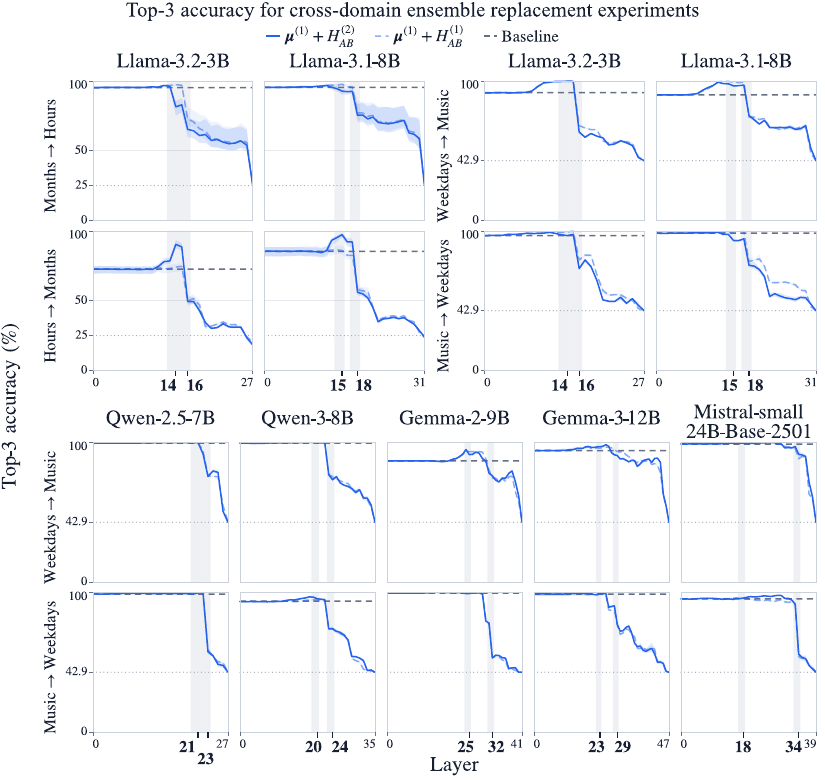}
    \caption{Top-3 accuracy for direct cross-domain replacement using \(H_{AB}^r\). Across all models and cross-domain pairs, the transplanted domain-2 interaction reproduces the layerwise behavior of the corresponding within-domain replacement.}
    \label{fig:app-top-3-cross-domain-replacement}
\end{figure}

\begin{figure}
    \centering
    \includegraphics[width=\linewidth]{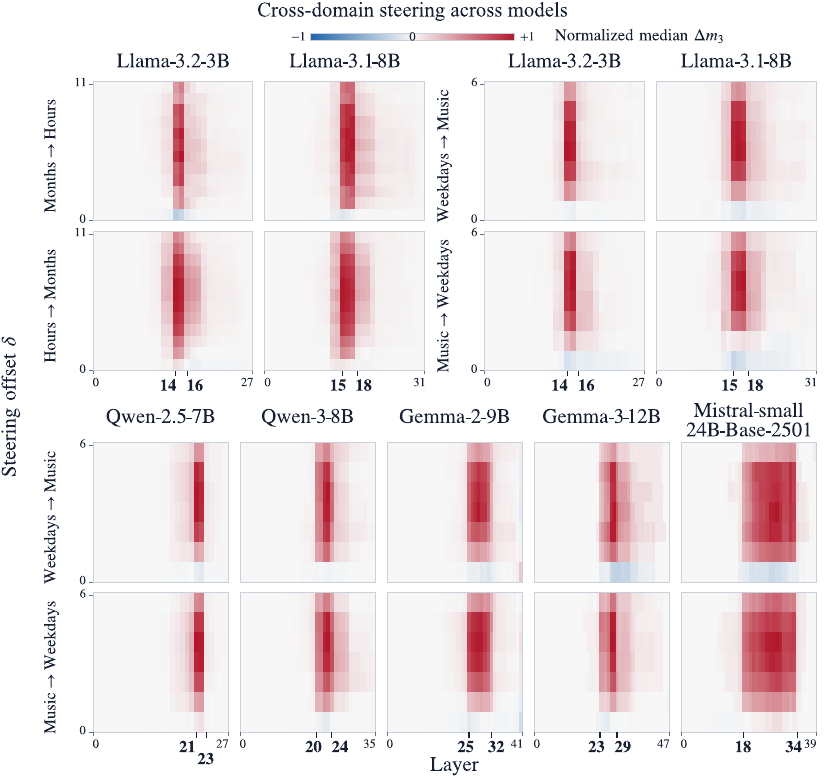}
    \caption{Cross-domain steering using vectors extracted from domain-2. Color indicates normalized median \(\Delta m_3\). Steering is strongest in the same intermediate layers where \(H_{AB}^r\) is causally relevant. The \(\delta=0\) case provides a non-trivial control: replacing the domain-1 relation vector with the corresponding domain-2 vector should preserve the original answer if the relation representation transfers across domains.}
    \label{fig:app-cross-domain-steering}
\end{figure}

\begin{figure}
    \centering
    \includegraphics[width=\linewidth]{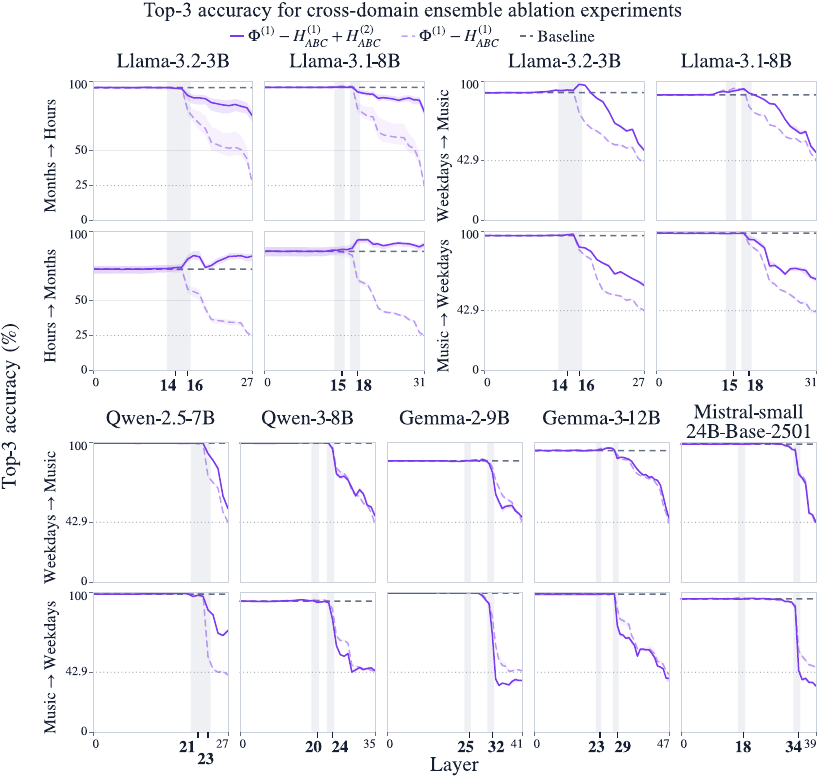}
    \caption{Top-3 accuracy for cross-domain ablation transplant of \(H_{ABC}^r\). Unlike \(H_{AB}^r\), direct transfer of the three-way interaction is not consistently successful across domains. Recovery in performance is observed for the months\(\leftrightarrow\)hours transfers in the Llama models, but most other model–domain pairs show little improvement over ablation.}
    \label{fig:app-top-3-cross-domain-habc-ablation}
\end{figure}

\begin{figure}
    \centering
    \includegraphics[width=\linewidth]{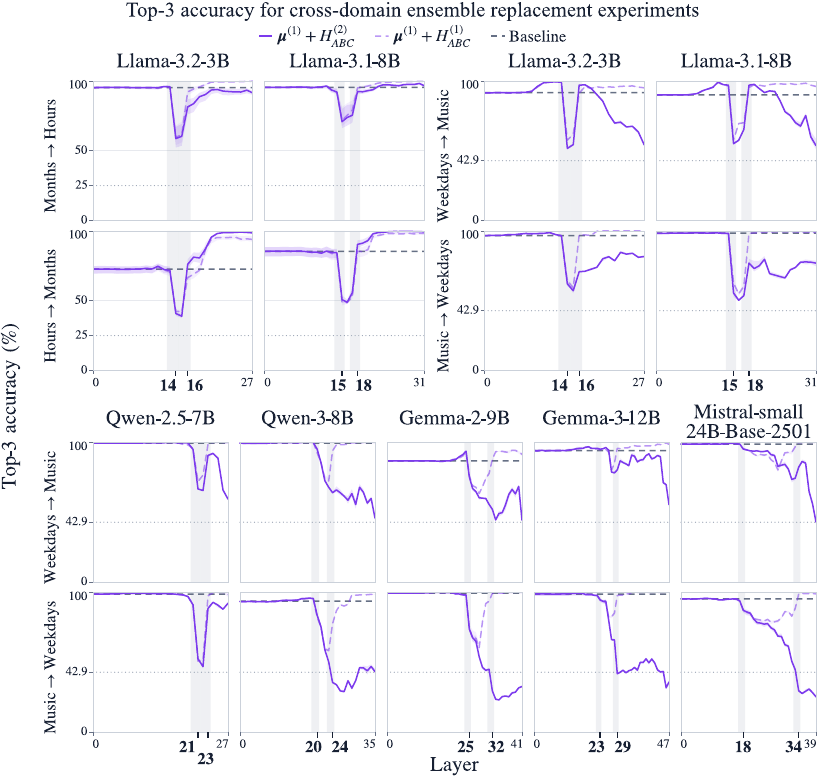}
    \caption{Top-3 accuracy for cross-domain replacement using \(H_{ABC}^r\). Consistent with Figure~\ref{fig:app-top-3-cross-domain-habc-ablation}, the ability of a domain-2 \(H_{ABC}^r\) to substitute for its domain-1 counterpart depends strongly on the model and domain pair.}
    \label{fig:app-top-3-cross-domain-habc-replacement}
\end{figure}

\begin{figure}
    \centering
    \includegraphics[width=\linewidth]{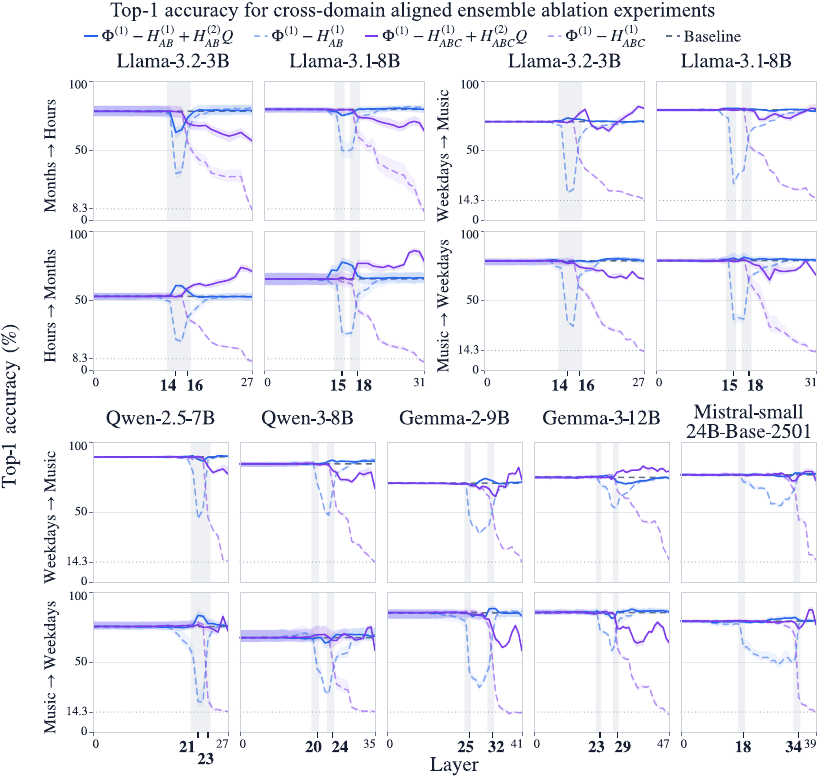}
    \caption{Top-1 accuracy for cross-domain ablation transplant after aligning the interaction subspaces using training replicates and evaluating on held-out replicates. Solid curves show aligned domain-2 interaction vectors transplanted into domain 1; dashed curves show the corresponding ablation and unmodified baselines. Alignment improves transfer for both \(H_{AB}^r\) and \(H_{ABC}^r\), although the degree of recovery varies across models and domain pairs.}
    \label{fig:app-top-1-aligned-ablation}
\end{figure}

\begin{figure}
    \centering
    \includegraphics[width=\linewidth]{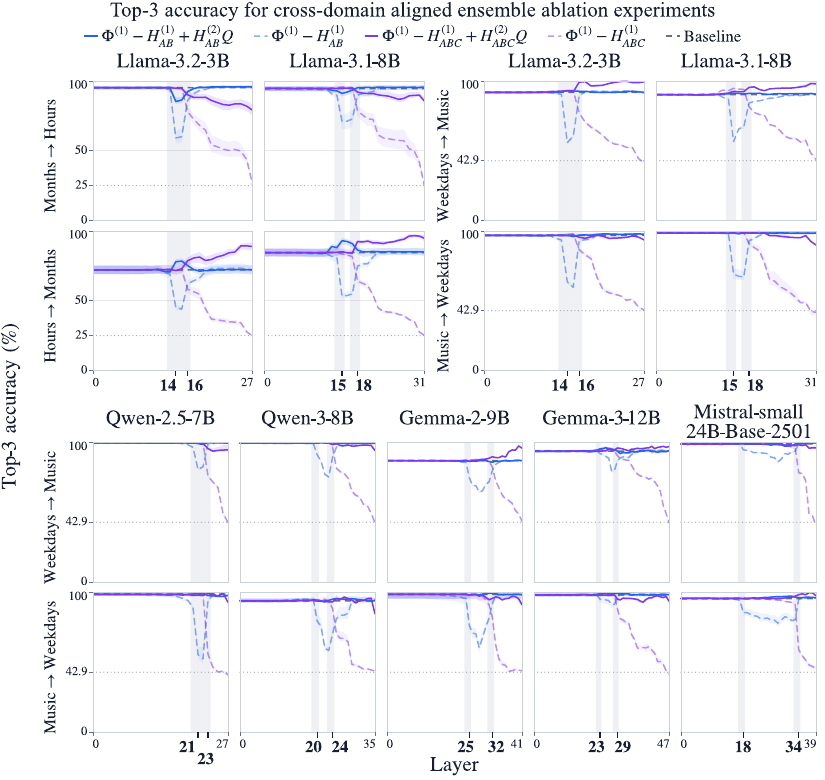}
    \caption{Top-3 accuracy for cross-domain ablation transplant after interaction-subspace alignment. Alignment improves recovery over the ablated baseline for both \(H_{AB}^r\) and \(H_{ABC}^r\), with the largest qualitative change occurring for \(H_{ABC}^r\), whose direct transfer was often weak.}
    \label{fig:app-top-3-aligned-ablation}
\end{figure}

\begin{figure}
    \centering
    \includegraphics[width=\linewidth]{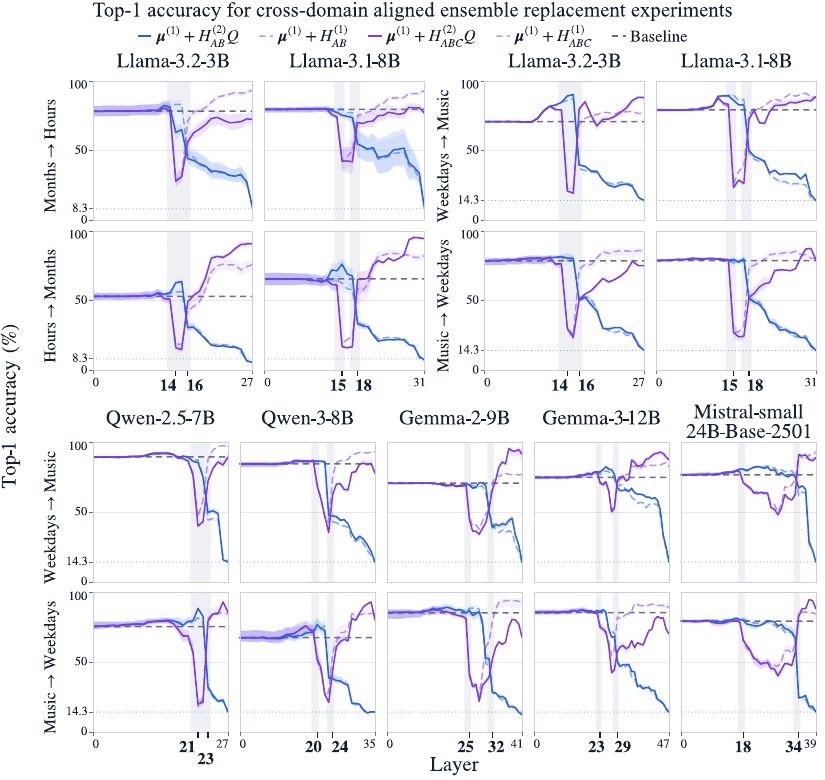}
    \caption{Top-1 accuracy for cross-domain replacement after aligning the interaction subspaces. Solid curves retain the domain-1 mean together with an aligned domain-2 interaction vector; dashed curves show the corresponding within-domain replacement. Alignment allows the transplanted interaction to reproduce the within-domain replacement trajectory more closely, particularly for \(H_{ABC}^r\).}
    \label{fig:app-top-1-aligned-replacement}
\end{figure}

\begin{figure}
    \centering
    \includegraphics[width=\linewidth]{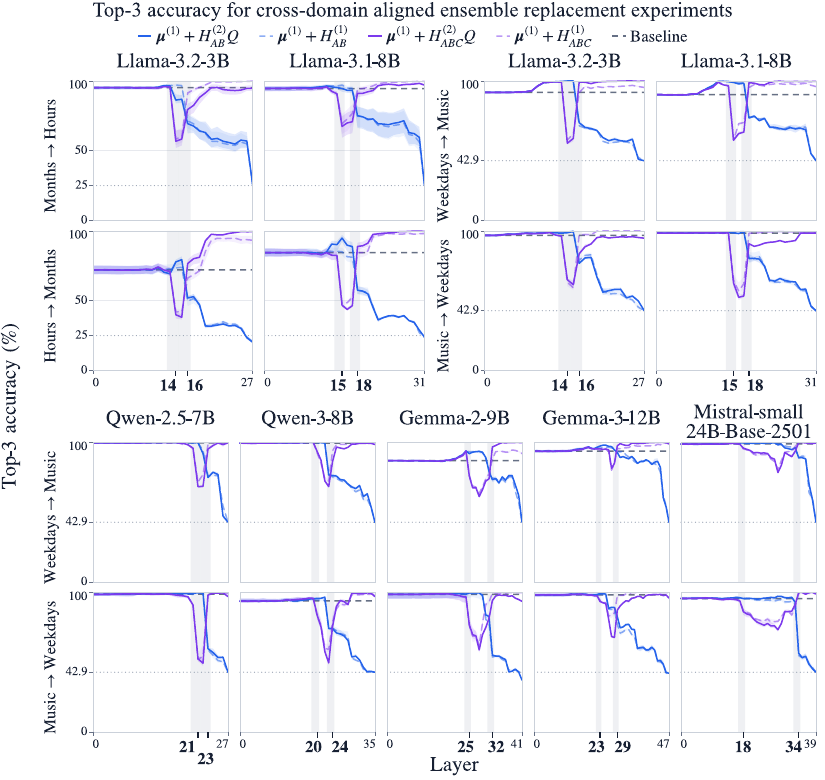}
    \caption{Top-3 accuracy for cross-domain replacement after interaction-subspace alignment. As in the top-1 results in Figure~\ref{fig:app-top-1-aligned-replacement}, alignment improves agreement between cross-domain and within-domain replacement, especially for \(H_{ABC}^r\), although the degree of recovery depends on the model and domain.}
    \label{fig:app-top-3-aligned-replacement}
\end{figure}
\end{document}